\documentclass{article}

\usepackage{arxiv}

\usepackage[utf8]{inputenc} 
\usepackage[T1]{fontenc}    
\usepackage{hyperref}       
\usepackage{url}            
\usepackage{booktabs}       
\usepackage{amsfonts}       
\usepackage{nicefrac}       
\usepackage{microtype}      
\usepackage{lipsum}
\usepackage{times}
\usepackage[ruled]{algorithm2e}
\usepackage{multicol}
\usepackage{amsmath,amssymb,amsthm}
\usepackage{natbib}
\setcitestyle{authoryear,open={(},close={)}}
\usepackage{dsfont}
\usepackage{xpatch}
\usepackage{multicol}
\usepackage{graphicx}
\newcommand{\ex}{{\rm e}\,}

\xpatchcmd{\proof}{\itshape}{\normalfont\proofnamefont}{}{}
\newcommand{\proofnamefont}{\bfseries}

\newtheorem{proposition}{Proposition}
\newtheorem{corollary}[proposition]{Corollary}
\newtheorem{lemma}[proposition]{Lemma}

\newtheorem{example}{Example}

\newtheorem{theorem}[proposition]{Theorem}
\newtheorem{definition}{Definition}

\def\R{\mathbb{R}} 
\def\1{~\mbox{I\hspace{-.6em}1}} 
\def\E{\mathbb{E}} %

\def\P{\mathbb{P}} %
\def\v{\mbox{Var\,}}

\title{Stochastic Online Convex Optimization. Application to probabilistic time series forecasting.}
\usepackage{times}

\author{%
Olivier Wintenberger\\
{\tt olivier.wintenberger@upmc.fr}\\
Laboratoire de Probabilit\'es, Statistique et Mod\'elisation\\
       Sorbonne Universit\'e, CNRS\\
       4 place Jussieu, 75005 Paris, France
}

\begin{document}

\maketitle

\begin{abstract}
We introduce a general framework of stochastic online convex optimization to obtain fast-rate stochastic regret bounds. We prove that algorithms such as online newton steps and a scale-free 10 version of Bernstein online aggregation achieve best-known rates in  unbounded stochastic settings. We apply our approach to calibrate parametric probabilistic forecasters of non-stationary sub-gaussian time series. Our fast-rate stochastic regret bounds are any-time valid. Our proofs combine self-bounded and Poissonnian inequalities for martingales and sub-gaussian random variables, respectively, under a stochastic exp-concavity assumption.
\end{abstract}


\section{Introduction}

We introduce a stochastic version of the Online Convex Optimization (OCO) analysis of \cite{zinkevich2003online} to calibrate sequential parametric forecasters and measure their performances in stochastic environments. Let ${\cal K}$ be a convex body of $\R^d$, i.e. a convex, compact set with a non-empty interior, and $\ell_t $, $t\ge1$, be loss functions  from ${\cal K}$ to $\R$. In Stochastic Online Convex Optimization (SOCO) analysis the loss functions $\ell_t$ are random elements. We consider a filtration $(\mathcal F_t)$ of non-decreasing $\sigma$-algebras such that the sequential learning algorithm predictions $(x_{t+1})$ and the losses $(\ell_{t})$ are $\mathcal F_{t}$-adapted. We measure  the performances of the sequentialalgorithm with the stochastic regret
\begin{equation}\label{eq:regret}
Regret_T= \sup_{x\in{\cal K}}\Big\{\sum_{t=1}^TL_t(x_t)-\sum_{t=1}^T L_t(x)\Big\}, \qquad T\ge 1\,,
\end{equation}
where $L_t(x_t)$ is the conditional risk 
$
L_t(x_t)=\E[\ell_t(x_t)\mid \mathcal F_{t-1}]$, $t\ge 1.$
Since the stochastic regret is random, we focus on any-time valid deviation rates such that, with high probability, it holds
\[
Regret_T\le O(\sqrt{ T})\mbox{ or }O(\log T)\mbox{ or }O((\log\log T)^2)\,,\qquad \text{for all }T\ge 1\,.
\]
The stochastic regret coincides with the one of OCO
\[
\sup_{x\in{\cal K}}\Big\{\sum_{t=1}^T \ell_t(x_t)-\sum_{t=1}^T \ell_t(x)\Big\}
\]
if the distributions of the loss functions $\ell_t$ are Dirac masses for every $t\ge0$. Thus SOCO analysis encompasses the classical deterministic OCO analysis.

However, the two analyses differ in various ways. First, the stochastic environment can improve the convex properties of the optimization problem; The conditional risk functions often have better convex properties than the loss functions.
Second, the competitors in stochastic and deterministic regret bounds are not the same; The conditional risks can measure the calibration of parametric probabilistic forecasters to the conditional distributions of the environment.
Third, it is likely that the maximum of the deviations of the random loss functions around the conditional risks increase with the number of iterations. The sequential algorithms should be robust to these deviations.

We use the SOCO analysis to prove that some parametric forecasters are robust to sub-gaussian stochastic environments when calibrated sequentially. Our first main result in Section \ref{sec:ons} states that the calibration using the Online Newton Step (ONS) algorithm achieves a $O(\log T)$ stochastic regret bound for any conditionally sub-gaussian sequence of random losses. The fundamental assumption is a stochastic exp-concavity condition {\bf (H2)} that holds for non-convex losses and unbounded gradients. The proof uses a self-normalized martingale inequality, and a Poissonnian inequality valid for conditional sub-gaussian gradients as in Condition {\bf(H3)}.  Our study gives insights why the use of second-order gradient algorithms such as ONS yields a fast-rate calibration; ONS implicitly minimizes a surrogate loss involving second-order terms. 

Then we extend the deterministic expert aggregation analysis introducing the Stochastic Online Aggregation (SOA) analysis. In SOA, the experts are  stochastic predictors adapted to the filtration $\mathcal F_{t-1}$, and the aggregation algorithm competes with the best predictor. The best existing regret bounds  achieve optimal rates $O(\log\log T +E)$ in any stochastic environment  bounded by $E>0$. However,  $E$ being the maximum of stochastic deviations, it usually increases as $O(\sqrt{\log(T)})$  and deteriorates the rate in a sub-gaussian environment. 

Our second result is a stochastic regret bound with rate  $O((\log \log T)^2)$ achieved by a scale-free version of the algorithm Bernstein Online Aggregation (BOA); We tune the multiple learning rates such that the weights are insensitive to the multiplication of the losses by a scalar. This property is crucial in the proof to deal with stochastic losses and the obtained regret bound improves the existing ones in some unbounded stochastic settings. 

In Section \ref{sec:boaons} we show that we can use the SOCO analysis to calibrate parametric probabilistic forecasters. We consider gaussian probabilistic forecasters of time series and logarithmic losses such that the conditional risk functions coincide with the Kullback-Leibler (KL) divergence. We interpret the stochastic regrets bounds as cumulative KL bounds relative to a static optimal forecaster.

We verify the condition {\bf (H2)}  on parametric gaussian forecasters of a time series $(y_t)$. Then we apply SOCO to parametric forecasters using AR-ARCH modeling to predict the conditional expectations and variances. Even though the corresponding logarithmic loss functions are not convex, the conditional risk functions are still locally stochastically exp-concave. Thus we can combine ONS and BOA algorithms to sequentially calibrate the parameters of the gaussian probabilsitic forecasters.  We provide fast-rate non-asymptotic theoretical guarantees for such parametric probabilistic forecasters.


The stochastic regret bound \eqref{eq:regret} is obtained using \cite{ville1939etude}'s inequality  and is any-time valid. Any-time valid sequential inference have been recently applied with success in \cite{henzi2022valid,shafer2021testing,waudby2020estimating} to many statistical problems such as testing, comparing forecasters and designing confidence sequences.  We refer to the textbook \cite{shafer2019game} and the survey paper  \cite{ramdas2022game} for an exhaustive overview. Sequentially calibrated non-parametric probabilistic forecasters are developed in Chapter 12 of  \cite{shafer2019game} with a $O(\sqrt{T})$ regret bound in any bounded stochastic environment. Faster $O(\log T)$ regret bounds for parametric prediction of deterministic individual sequences are presented in \cite{cesa2006prediction,hazan2016introduction} under exp-concavity assumptions. 

For independent and identically distributed (iid) loss functions $\ell_t$, \cite{hazan2016introduction,mahdavi2015lower} proved that Online Gradient Descent (OGD) and ONS algorithms satisfy stochastic regret bounds of order $O(\sqrt{ T})$ and $O(\log T)$, respectively. Expert aggregation calibrated by Squint and  BOA achieves a $O(\log\log T)$ stochastic regret bound under the so-called Bernstein condition in the stationary bounded setting, see \cite{koolen2016combining} and \cite{wintenberger2017optimal}, respectively. These results have been improved by a careful tuning of the learning rate in \cite{mhammedi2019lipschitz,orseau2021isotuning}. All existing stochastic regret bounds have a linear dependence in the maximum of the deviations of the loss functions. They all use a fast-rate "online to batch" conversion to turn deterministic regret bounds into stochastic ones, see \cite{mehta2017fast}. We adopt a different approach introducing surrogate losses and achieving results far beyond the iid environment.

Sequential learning naturally applies to time series as recursive algorithms update their predictions when observing new data over time. However, regret bounds with high probability are rare due to the data's temporal dependence that prevents the use of standard exponential inequalities. For stationary $\beta-$ or $\phi-$mixing time series, \cite{agarwal2012generalization} obtained fast-rate regret bounds for the unconditional risk function $\E[\ell_t]$. \cite{anava2013online} obtained fast-rate regret bounds for the ONS algorithm risk for ARMA (Auto-Regressive Moving-Average) models. Their notion of stochastic regret does not coincide with ours. 

We obtain fast-rate sequential calibration   combining optimization (ONS) and aggregation (BOA). Our strategy shares similarities with existing algorithms developed by \cite{giraud2015aggregation,van2021metagrad}. Such algorithms achive fast-rate of stochastic regret bound in some stationary environments \citep{koolen2016combining}. The algorithm developed by \cite{adjakossa2020kalman} aggregates Kalman recursions in non-stationary well-specified settings only. Finally, sequential  algorithms for estimating the volatilities or aggregating probability forecasters have been recently developed by \cite{werge2022adavol} and \cite{thorey2017online}, \cite{v2019online}, respectively.

\section{Preliminaries and assumptions}\label{sec:prel}
We use the notation ${\bf0}=(0,\ldots,0)^T$, $\1=(1,\ldots,1)^T$, and the operations implying vectors are thought componentwise.  In the sequel $\|\cdot\|$ is the Euclidian norm $\|\cdot\|_2$.
We consider a filtration $(\mathcal F_t)$, $t\ge0$, and $\mathcal F_0=\{\emptyset,\Omega\}$ by convention. The proofs of the main results are deferred to Appendix \ref{sec:proofs}.
\begin{definition}[Stochastic online convex optimization]
Consider a convex body ${\cal K}\subset \R^d$ and an $\mathcal F_t$-adapted sequence of random loss functions $(\ell_t)$ defined over ${\cal K}$. An algorithm predicts $x_t\in{\cal K}$ that is $\mathcal F_{t-1}$-measurable and incurs the random conditional risk $L_t(x_t)=\E_{t-1}[\ell_t(x_t)]$ at each step $t\ge 1$. SOCO analyses the rate of the stochastic regret \eqref{eq:regret} as a function of $T\ge 1$ assuming  the risk functions $L_t$ being convex for all $t\ge1$. 
\end{definition}
The main difference with the classical OCO analysis is the use of the conditional risk functions $L_t$ instead of the loss functions $\ell_t$ in the regret and the convex assumption. The SOCO setting extends the OCO setting. 
\begin{proposition}\label{prop:ext}
Any OCO problem is a degenerate SOCO problem.
\end{proposition}
\begin{proof}
We consider that $\ell_t$ has a degenerate distribution $\delta_{\{\ell_t\}}$, the Dirac mass at $\ell_t$. It is a SOCO problem equipped with the natural filtration is ${\cal F}_{t}=\{\emptyset,\Omega\}$, $t\ge1$, and $L_t=\ell_t$. \end{proof}
%
The conditional distribution of the random loss function $\ell_t$ may depend adversarially on $x_t,\ldots,x_1\in{\cal F}_{t-1}$, and, as in OCO, a boundedness assumption on $\cal K$ is necessary to obtain regret bounds.\\

{\bf (H1)} The diameter of ${\cal K}$ is $D<\infty$ so that $\|x-y\|\le D$, $x,y\in\mathcal K$, and the loss functions $\ell_t$ are continuously differentiable over $\mathcal K$ a.s. with integrable gradients. \\

Under {\bf (H1)} and if the loss functions $(\ell_t)$ are convex the optimal rate is $O(\sqrt T)$ for OCO and thus for SOCO  problems  by an application of Proposition \ref{prop:ext}. This optimal rate is satisfied in SOCO problems even if the loss functions $(\ell_t)$ are not convex but the risk functions $(L_t)$ are. See Appendix \ref{sec:ogd} for the case of the OGD when the gradients $\nabla\ell_t$ are a.s. bounded by $G>0$. To obtain fast-rate $o(\sqrt T)$ stochastic regret bounds, we assume stochastic exp-concavity.\\

{\bf (H2)} The random loss functions $\ell_t$, $t\ge1$, are stochastically exp-concave if for some $\alpha\ge 0$: 
\[
L_t(y)\le L_t(x)+\nabla L_t(y)^T(y-x)-\frac\alpha2 \E_{t-1}[(\nabla \ell_t(y)^T(y-x))^2]\,,\qquad x,y\in\mathcal K\,,a.s.,t\ge1\,.
\]

Condition {\bf (H2)} with $\alpha=0$ coincides with the convexity assumption on $L_t$, $t\ge 1$. Also {\bf (H2)} with  $\alpha\ge 0$ does not imply the convexity of $\ell_t$, $t\ge 1$.  In the iid setting, stochastic exp-concavity has been studied  by \cite{koolen2016combining}, making explicit a condition introduced in \cite{rigollet2008learning}.
Condition {\bf (H2)} was used by \cite{gaillard2018efficient} over the unit $\ell^1$-ball and it implies the Bernstein condition of \cite{van2021metagrad} introduced for convex losses.
In the deterministic setting, an application of Lemma 4.3 of \cite{hazan2016introduction} shows that Condition {\bf (H2)} with $\alpha=1/2(\mu\wedge 1/(GD))$ follows from the $\mu$-exp-concavity of the loss functions. 
\begin{proposition}\label{prop:alpha}
Assume the loss functions are twice continuously differentiable. Then Condition {\bf (H2)} implies
\begin{equation}\label{eq:cond}
\alpha \E_{t-1}[\nabla \ell_t(x)\nabla \ell_t(x)^T]\preceq \nabla^2 L_t(x),\qquad x\in {\cal K}\,,a.s., t\ge 1\,.
\end{equation}
On the opposite, if $L_t$ is $\mu$-strongly convex and there exists $g>0$ such that
\begin{equation}\label{eq:condbis}
\E_{t-1}[\nabla \ell_t(x)\nabla \ell_t(x)^T]\preceq g^2I_d\,, \qquad x\in {\cal K}\,,a.s., t\ge1\,,
\end{equation}
then Condition {\bf (H2)} holds with $\alpha=\mu/g^2$.
\end{proposition}
\begin{proof}
Inequalities \eqref{eq:cond} and \eqref{eq:condbis} follow easily from a second-order Taylor expansion of $L_t$.
\end{proof}
We verify Condition {\bf (H2)} when calibrating parametric gaussian probabilistic forecasters in Section \ref{sec:boaons}. Under exp-concavity assumptions, the optimal rate is $O(\log T)$ in OCO \citep{hazan2016introduction} and thus in SOCO. In stochastic environments the constant $\alpha>0$ depends on the conditional distributions of the losses and is unknown in practice.

We consider unbounded sug-gaussian gradients introducing the Orlicz function 
$\psi_2(x)=\exp(x^2)-1\,,$ $x\in \R$.
Conditional sub-gaussian random variables are such as the Orlicz norm 
\[
\|Y_t\|_{\psi_2,t}=\inf\{c>0\,;\,\E_{t-1}[\psi_2(Y_t/c)]\le 1\,a.s.\}
\]
is bounded by a constant for every $t\ge 1$. This norm is not precise enough for our purpose. We require a slightly more explicit condition involving two constants. Our assumption is a conditional version of the Bernstein condition, also related to the notion of Bernstein-Orlicz norm of \cite{van2013bernstein}.\\

{\bf (H3)} The gradients $\nabla \ell_t(x_t)$, $t\ge1$, satisfy for $G_{\psi_2}$, $G_2>0$, and all $k\ge 1$, $t\ge 1$, $x\in {\cal K}$,\\[-.3cm]
\begin{align*}
\E_{t-1}[(\nabla \ell_t(x_t)^T(x_t-x))^{2k}]&\le k!(G_{\psi_2}D)^{2(k-1)}\E_{t-1}[(\nabla \ell_t(x_t)^T(x_t-x))^2]\qquad a.s.,\\
\E_{t-1}[\|\nabla \ell_t(x_t)\|^{2k}]&\le k!G_{\psi_2}^{2(k-1)}\E_{t-1}[\|\nabla \ell_t(x_t)\|^2]\qquad a.s.,\\
\E_{t-1}[\|\nabla \ell_t(x_t)\|^2]&\le G_2^2\qquad a.s.
\end{align*}

If the gradients $\nabla \ell_t(x_t)$, $t\ge1$, verify the condition {\bf (H3)} then they are conditionally sub-gaussian. \begin{proposition}\label{prop:pois}
Assume that the gradient $\nabla \ell_t(x_t)$ satisfies Condition {\bf (H3)}: then $\|\nabla \ell_t(x_t)\|$ is conditionally sub-gaussian with
\[
\max_{t\ge 1}\|\nabla \ell_t(x_t)\|_{\psi_2,t}\le 2(G_{\psi_2}\vee G_2)^2\,,\qquad t\ge 1\,,\qquad a.s. 
\]
\end{proposition}
\begin{proof}
Denote $Y=\|\nabla \ell_t(x_t)\|$. We have
\[
\E[\exp(Y^{2}/K)]\le 1+\sum_{k=1}^\infty \dfrac{\E[Y^{2k}]}{k!K^k}\le 1+\sum_{k=1}^\infty \dfrac{G_{\psi_2}^{2(k-1)}G_2^2}{K^k}\le 2
\]
for $K=2(G_{\psi_2}\vee G_2)^2$. We conclude by definition of the Orlicz' norm.
\end{proof}
Condition {\bf (H3)} is satisfied in every bounded cases $ \|\nabla \ell_t(x_t)\|^2\le G^2$, $t\ge 1$, with $G_{\psi_2}=G_2=G$, Thus our sub-gaussian stochastic setting encompasses the classical bounded  deterministic one. Condition {\bf (H3)} is also verified for unbounded gaussian gradients with second-order conditional moments  bounded by the constant $G_2>0$. Condition {\bf (H3)}  is independent of the conditional risks $\nabla L_t(x_t)=\E_{t-1}[\nabla \ell_t(x_t)]$, $t\ge 1$, and it does not interfere with Condition {\bf (H2)}.
\begin{proposition}\label{prop:gauss}
Assume that the gradient $\nabla \ell_t(x_t)$ is normally distributed given $\mathcal F_{t-1}$. Then Condition {\bf (H3)} is satisfied if $\E_{t-1}[\|\nabla \ell_t(x_t)\|^2]\le G_2^2$ a.s., $t\ge 1$, and then 
$G_{\psi_2}=8.5\, G_2\,.$
\end{proposition}

\section{ONS achieves fast-rate stochastic regrets}\label{sec:ons}
\subsection{Surrogate losses}
We base our approach on an observed surrogate loss that upper-bounds the stochastic regret using an exponential inequality for martingales from \cite{bercu2008exponential} on unbounded gradients $\nabla \ell_t$, $t\ge 1$.
\begin{proposition}\label{prop:surrogate}
Under {\bf (H1)} and {\bf (H2)}, for any  predictable sequence $(x_t)$ and deterministic $x$ in $\cal K$, it holds with probability $1-\delta$,  $0<\delta\le1$,
\begin{align*}
\sum_{t=1}^TL_t(x_t)-\sum_{t=1}^TL_t(x)\le& \sum_{t=1}^T\nabla \ell_t(x_t)^T(x_t-x)+\frac\lambda2\sum_{t=1}^T(\nabla \ell_t(x_t)^T(x_t-x))^2\\
&\hspace{-2cm}+\frac{\lambda-\alpha}2\sum_{t=1}^T\E_{t-1}[(\nabla \ell_t(x_t)^T(x_t-x))^2]+\frac2\lambda\log(\delta^{-1})\, \qquad \lambda>0,\,T\ge 1.
\end{align*}
\end{proposition}

When the distributions of $\ell_t$ are degenerate the upper bound in  Proposition \ref{prop:ext} becomes
\[
\sum_{t=1}^T\nabla \ell_t(x_t)^T(x_t-x)+\dfrac{2\lambda-\alpha}2\sum_{t=1}^T(\nabla \ell_t(x_t)^T(x_t-x))^2+\frac2\lambda\log(\delta^{-1})\,.
\] 
Since the result is valid with probability $1$, the last term disappears letting $\delta \uparrow 1$. Forthcoming results, any-time valid with a high probability in a stochastic environment, are surely valid in deterministic environments when suppressing the dependence in $\delta$.

Following \cite{van2021metagrad}, we interpret 
\[\widetilde \ell_t(x_t)=\nabla \ell_t(x_t)^T(x_t-x)+\frac\lambda2(\nabla \ell_t(x_t)^T(x_t-x))^2\,, \qquad t\ge 1\,,
\] 
as a surrogate loss. 
The quadratic term in addition to the gradient term is necessary to upper-bound the unobserved conditional risk with high probability. In stochastic environments, algorithms should minimize the cumulative surrogate loss $\sum_{t=1}^T\widetilde \ell_t$ rather than the cumulative loss $\sum_{t=1}^T \ell_t$. Under Condition {\bf (H2)} with $\alpha>0$, this additional quadratic term   is counterbalanced by the compensator $\sum_{t=1}^T\E_{t-1}[(\cdots)^2]$  when $\lambda <\alpha/2$. The Poissonnian inequality of Proposition \ref{prop:pois} relates both quadratic terms.

\subsection{The stochastic regret analysis of ONS}

The cumulative surrogate losses $\sum_{t=1}^T\widetilde \ell_t$ is implicitly minimized in the ONS's regret analysis of  \cite{hazan2016introduction}. Then the ONS algorithm achieves a fast stochastic regret bound.  

\begin{algorithm}[h]
\caption{Online Newton Step \citep{hazan2011beyond}}
    \label{alg:ons}
    {\bfseries Parameter:} $\gamma>0$.\\
{\bfseries Initialization:} Initial prediction $x_1\in {\cal K}$ and $A_0=\frac1{(\gamma D)^2} I_d$.\\
 {\bf Predict:} $x_t$ \\
 {\bf Incur:} $L_t(x_t)$\\
  {\bf Observe:} $\nabla \ell_t(x_t)\in\R^d$\\
    {\bfseries Recursion:}     Update  
\begin{align*}
A_{t}& = A_{t-1} +\nabla \ell_t(x_t)\nabla \ell_t(x_t)^T,\\
y_{t+1} &= x_{t} - \gamma^{-1} A_{t}^{-1} \nabla \ell_t(x_t) \,,\\
x_{t+1}&=\arg\min_{x\in{\cal K}}(x-y_{t+1})^TA_t(x-y_{t+1})\,,\qquad \text{projection step}.
\end{align*}
\end{algorithm}

Using the Sherman-Morrison formula, each step of ONS has a $O(d^2+P)$-cost, where $P$ is the cost of the projection step.
If the gradients $\nabla \ell_t(x_t)$, $t\ge1$, verify the condition {\bf (H3)} then the square of their Euclidian norm $\|\nabla \ell_t(x_t)\|^2$ satisfies   a Poissonian exponential inequality.
\begin{proposition}\label{prop:pois}
Under Condition {\bf (H3)} the gradients $\nabla \ell_t(x_t)$ satisfy
\[
\E_{t-1}[\exp(\eta(\|\nabla \ell_t(x_t)\|^2-\E[\|\nabla \ell_t(x_t)\|^2]/(1-\eta G_{\psi_2}^2))]\le 1\,,\qquad \forall \eta ^2<1/G_{\psi_2}\,,\qquad t\ge 1\,,a.s. 
\]
\end{proposition}
\begin{proof}
Expanding the exponential and using Condition ({\bf H3}) we obtain
\begin{multline*}
\E[\exp(\eta Y^2)]=\sum_{k=0}\dfrac{\eta^k \E[Y^{2k}]}{k!}\le 1 +\eta\E[Y^2]\Big(1+\sum_{k\ge 2}\eta^{k-1}G_{\psi_2}^{2(k-1)}\Big)\\=1+\dfrac{\eta\E[Y^2]}{1-\eta G_{\psi_2}^2}\le \exp(\eta\E[Y^2]/(1-\eta G_{\psi_2}^2))
\end{multline*}
for every $\eta G_{\psi_2}^2<1$ and the desired result follows.
\end{proof}
To control the second-order terms in Proposition \ref{prop:surrogate}, we combine the self-bounded martingale and Poissonian inequalities. We obtain a fast-rate stochastic regret bound for the ONS tuned choosing $\gamma=\alpha/3$. \begin{theorem}\label{th:ons}
Under {\bf (H1)}, {\bf (H2)} and {\bf (H3)}, the ONS algorithm \ref{alg:ons} for $\gamma= \alpha/3$ satisfies with probability $1-3\delta$ the stochastic regret bound
\[
Regret_T\le\frac3{2\alpha}\Big(1+d\log\Big(1+\dfrac{2\alpha^2 D^2(TG_2^2+G_{\psi_2}^2\log(\delta^{-1}))}9\Big)\Big)
 +\Big(\frac{4\alpha(G_{\psi_2}D)^2}9+\frac{18}{\alpha}\Big)\log(\delta^{-1})
\]
valid for every $T\ge 1$.
\end{theorem}
Our result extends fast-rate stochastic regret bounds for ONS far beyond existing results in the iid bounded setting.

\section{BOA achieves fast-rate regret bounds in Stochastic Online Aggregation}\label{sec:boa}
\subsection{Stochastic Online Aggregation}\label{sec:soa}
We consider ${\mathbf x}_t=[x_t^{(1)},\ldots,x_t^{(K)}]$ a $d\times K$ matrix whose columns are $K$ different $\mathcal F_{t-1}$-adapted predictors $x_t^{(i)}$. We denote $\widehat x_t= {\mathbf x}_t\pi_t=\sum_{i=1}^K\pi_i x_t^{(i)}$ their aggregation, with $\pi_t$ in the simplex $\Lambda_K=\{\pi\in \R^K;\, \pi>{\bf 0}, \sum_{i=1}^K\pi_i=1\}$. Aggregation algorithms combine the predictors with weights $\pi_t$ minimizing the stochastic regret
\[
Regret_T^{ag}=\max_{1\le i\le K}\Big\{\sum_{t=1}^TL_t\big(\widehat x_t\big)-\sum_{t=1}^T L_t\big(x^{(i)}\big)\Big\},\qquad T\ge 1\,.
\]

We have under Condition {\bf (H2)} the relation
\[
L_t(\widehat x_t)-L_t(\mathbf x_t\pi)\le\nabla L_t(\mathbf x_t\pi_t)^T\mathbf x_t(\pi_t-\pi)-\frac\alpha2 \E_{t-1}[(\nabla \ell_t(\mathbf x_t\pi_t)^T\mathbf x_t(\pi-\pi_t))^2]\,.
\]
We consider the loss functions $\pi\to \ell_t(\mathbf x_t\pi)$ over ${\cal K}=\Lambda_K$ that is stochastically exp-concave with the same constant $\alpha$ as the original loss functions $\ell_t$. Applying Proposition \ref{prop:surrogate} under Condition {\bf (H2)} we obtain 
\begin{align}\label{eq:surrog}
\sum_{t=1}^TL_t(\widehat x_t)-L_t(\mathbf x_t\pi)\le& \sum_{t=1}^T\nabla \ell_t(\mathbf x_t\pi_t)^T\mathbf x_t(\pi_t-\pi)+\frac\lambda2\sum_{t=1}^T(\nabla \ell_t(\mathbf x_t\pi_t)^T\mathbf x_t(\pi_t-\pi))^2\nonumber\\
&+\frac{\lambda-\alpha}2\sum_{t=1}^T\E_{t-1}[(\nabla \ell_t(\mathbf x_t\pi_t)^T\mathbf x_t(\pi_t-\pi))^2]+\frac2\lambda\log(\delta^{-1})\,.
\end{align}
We identify the surrogate losses
\[
(\pi_t-\pi)^T{\boldsymbol \ell}_t+\lambda/2((\pi_t-\pi)^T{\boldsymbol \ell}_t)^2
\]
with gradients denoted by $
{\boldsymbol\ell}_t= \mathbf x_t^T \nabla\ell_t(\mathbf x_t\pi_t)$. We analyze algorithms minimizing the sum of the surrogate losses in stochastic environments.
We compare the aggregation strategy $\widehat x_t$ to $\pi\in\{e_i,1\le i\le K\}$, i.e., with the best predictor $x_t^{(i)}$, using the linear losses $(\pi_t-\pi)^T{\boldsymbol \ell}_t$ over ${\cal K}=\Lambda_K$. We call this problem, encompassing the deterministic expert aggregation problem, the Stochastic Online Aggregation (SOA). 

\subsection{The stochastic regret for the scale-free version of BOA}\label{sec:boa}

 The version of BOA described in Algorithm \ref{alg:boa} is  different than the original BOA algorithm in \cite{wintenberger2017optimal}, because of the specific tuning of the multiple learning rates $\eta_t$. The specific $\eta_t$ provides a  self-normalization and the algorithm is scale-free, i.e., insensitive to a multiplicative factor of the losses.
 \begin{algorithm}
\caption{Bernstein Online Aggregation \citep{wintenberger2017optimal}, scale-free version}  
\label{alg:boa}
{\bfseries Initialization:} Initial   weights $\pi_1\in \Lambda_K$ and   $\eta_0^{-2}=\widetilde L_{0}={\bf 0}$ ($\in \R^K$).\\
 For each step $t\ge 1$: the predictors incur the losses $\boldsymbol \ell_t\in\R^K$.\\
{\bf Recursion:}  Update 
\begin{align*}
\eta_{t}^{-2}&=\eta_{t-1}^{-2}+2.2(\boldsymbol\ell_t-\pi_t^T\boldsymbol\ell_t\1)^2\,,\\
\widetilde L_{t}&=\widetilde L_{t-1} + (\boldsymbol\ell_t-\pi_t^T\boldsymbol\ell_t\1) +\eta_{t} (\boldsymbol\ell_t-\pi_t^T\boldsymbol\ell_t\1)^2\,, \\
\pi_{t+1} &=  \dfrac{\eta_{t}\exp(-\eta_{t}\widetilde L_{t})\pi_1}{\pi_1^T(\eta_{t}\exp(-\eta_{t}\widetilde L_{t}))}\,.
\end{align*}
\end{algorithm}

The factor 2.2 is not arbitrary and is chosen such as a small numeric constant satisfying
\[
 \exp\Big(-\dfrac{y}{\sqrt{1+2.2y^2}}-\dfrac{y^2}{1+2.2y^2}\Big) \le 1- \dfrac{y}{\sqrt{1+2.2y^2}}\,,\qquad y\in \R\,.
\]
This relation is crucial in the proof of Theorem \ref{eq:boa} to propagate the self-normalization in a recursive argument. The coordinate-wise learning rate $\eta_{t,i}$ is only well defined after the first non-null observation $\underline m_{i}:=(\boldsymbol\ell_{t,i}-\pi_t^T\boldsymbol\ell_t )\neq0$, $1\le i\le K$. Before that time $\widetilde L_{t,i}=\widetilde L_{t-1,i}=\cdots= 0$ by convention.
Contrary to the ONS, and thanks to the adaptive learning rates, the algorithm BOA is parameter-free as it does not require the knowledge of $\alpha$, and each step has a $O(K)$-cost. We provide a deterministic regret bound valid for any deterministic sequence.
\begin{theorem}\label{thm:boa}
For every $1\le i\le K$, the BOA algorithm \ref{alg:boa} achieves the deterministic regret bound
\begin{multline}\label{eq:boa}
\sum_{t=1}^T\pi_t^T\boldsymbol\ell_t-\sum_{t=1}^T\pi_t^T \boldsymbol\ell_{t,i}\le \sqrt{2.2\sum_{t=1}^T(\pi_t^T\boldsymbol\ell_t-\boldsymbol\ell_{t,i})^2}\Big(\frac1{1.1}+\log(\pi_{1,i}^{-1})\\
 +\sum_{i=1}^K 1\{\max_{2\le t\le T}x_{t,i}>1/4\big\}\log(1+(M_{T,i}/\underline m_i)^2) +\log\big(\ex+\frac12\pi_1^T\log\big(1+(M_{T}/\underline m)^2T\big)\big)\Big)\Big)
\end{multline}
where $x_T=\eta_{T-1}(\boldsymbol\ell_T-\pi_{T}^T\boldsymbol\ell_T\1)$, $M_T=\max_{2\le t\le T}|\boldsymbol\ell_t-\pi_t^T\boldsymbol\ell_t\1|\in\R^K$ and $\underline m_{i}$ is the first non null observation of $\boldsymbol\ell_{t,i}-\pi_t^T\boldsymbol\ell_t $.
\end{theorem}
The term 
\[
\sum_{i=1}^K1\{\max_{2\le t\le T}x_{t,i}>1/4\big\} \log(1+(M_{T,i}/\underline m_i)^2) 
\]
in the regret bound \eqref{eq:boa} replaces the term $\|M_T\|_\infty$ in the regret bounds of \cite{mhammedi2019lipschitz,orseau2021isotuning}. In some unbounded stochastic settings, our regret bound \eqref{eq:boa} is better for $T$ large. For instance, if we consider that the first predictor is iid standard gaussian and the other ones are bounded then $ \|M_T\|_\infty\sim M_{T,1} \sim \sqrt{2\log T}$ is much larger than $\sum_{i=1}^K \log(1+(M_{T,i}/\underline m_i)^2)\sim \log\log T$ for $T$ large. 

The deterministic regret bound in Theorem \ref{thm:boa} is assumption-free. Its first term is the square root of the sum of the additional quadratic terms in the surrogate losses \eqref{eq:surrog}. It may increase at the rate $O(\sqrt T)$ but, under condition {\bf(H2)}, it becomes negligible.
We provide a stochastic regret bound for sequential aggregation using BOA.
\begin{theorem}\label{th:boa}
Assume Conditions {\bf (H1)}, {\bf (H2)}  and {\bf (H3)} hold on $ \mathbf x_t^T \nabla\ell_t(\mathbf x_t\pi_t)$ a.s. for all $t\ge 1$, $1\le i\le K$. The scale-free BOA algorithm \ref{alg:boa} with $\pi_i\ge e^{-K}$ for all $1\le i\le K$ satisfies, with probability $1-3\delta$,\begin{align*}
Regret_T^{ag}&\le\dfrac{3(K+1)^2}\alpha \Big(\log\Big(1+\dfrac{2G_{\psi_2}^2\log T}{\underline m^ 2}\Big)\Big)^2\\&+O((\log\log\log T)^2)
+\Big(\dfrac{2\alpha}3(G_{\psi_2}D)^2+\frac{6}\alpha\Big)\log(\delta^{-1})\,.
\end{align*}
for every $T\ge 1$, and $\underline m>0$ such that $\P(\min_{1\le i\le K} \underline m_i\ge\underline m)\le 1-\delta$.
\end{theorem}
Aggregation problems are easier than optimization ones and BOA achieves a faster  stochastic regret bound than ONS. This rate $O((\log \log T)^2)$ is suboptimal in the deterministic expert aggregation setting. Condition {\bf (H3)} implies that the deterministic gradients are bounded by a constant $G>0$, and Condition {\bf(H2)} implies exp-concavity. Optimal strategies achieve $O(G\log K)$ deterministic regret \citep{cesa2006prediction}.  Among them Exponentially Weighted Aggregation, but this algorithm achieves only a $O(\sqrt T)$ stochastic regret as shown by \cite{audibert2007progressive}. Best-known aggregation algorithms in deterministic and unbounded stochastic settings are different. It is an open question to find an aggregation algorithm optimal in both settings whereas squint and the original version of BOA achieve optimal rates in bounded deterministic and stochastic settings. The choice of the initial weights $\pi_1$ being not crucial in the latter setting we choose implicitely uniform initial weights in the sequel.

\subsection{The SOCO analysis to adapt to unknown stochastic exp-concavity constant $\alpha>0$}\label{sec:adapt}
 We study an example of BOA-ONS dealing with the adaptation to the best stochastic exp-concavity constant $\alpha$. It is crucial for improving the ONS performances in any stochastic environment where, contrary to deterministic ones, there is no way to determine the optimal $\alpha$ as it depends on the conditional distributions of $\ell_t$.
Consider $\widehat x_t= {\mathbf x}_t\pi=\sum_{i=1}^K\pi_i x_t^{(i)}$ the BOA aggregation of $K\ge 1$ ONS predictions with different parameters $\gamma^{(i)}$ with $\gamma^{(i)}=\{2^{-1},\ldots,2^{-K}\}$. The resulting BOA-ONS algorithm adapts to the optimal value of $\alpha$ that depends on the unknown stochastic environment. The algorithm Metagrad of \cite{van2021metagrad} is also able to adapt to different rates of convergence. 
\begin{corollary}
Under {\bf (H1)}, {\bf (H2)} and {\bf (H3)} with $\alpha\ge 2^{-K-2}$, BOA-ONS algorithm satisfies with probability $1-4\delta$ the stochastic regret bound
\begin{align*}
\sum_{t=1}^TL_t(\widehat x_t)-\sum_{t=1}^TL_t( x)\le&\frac1{\alpha}O(d\log(T)+K^2\log\log(T)^2) +O\Big(\alpha(G_{\psi_2}D)^2 +\frac{1}\alpha\Big)\log(\delta^{-1})\,.
\end{align*}
\end{corollary}
 \begin{proof}
We combine the stochastic regret bound of Theorem \ref{th:boa} with the inequality \eqref{eq:ons} choosing $-\log_2(\gamma)+1\le i\le-\log_2(\gamma)+2$ so that $\alpha/4\le \gamma\le \alpha/2$ for $\alpha\le 1$.
\end{proof}

\section{BOA-ONS for sequential prediction and probabilistic forecast of time series}\label{sec:boaons}

\subsection{Probabilistic forecasting}

Observing a time series $(y_t)$,   we use the SOCO analysis to calibrate some parametric probabilistic forecasters in the sense of Chapter 12 of \cite{shafer2019game}. In our setting sequential algorithms predict $x_t$ and parametrize a probabilistic forecaster $P_{x_t}$. Given a scoring rule $S$,  the loss at step $t$ is $\ell_t(x_t)=S(P_{x_t},y_t)$ . The expected score, also denoted by $S$ in  \cite{gneiting2007strictly}, is a discrepancy measure between probabilities
\[
S(P_{x_t},P_{t})=L_t(x_t) =\E_{t-1}[S(P_{x_t},y_t)]\,,
\]
where $P_t$ denotes the distribution of $y_t$ given $\mathcal F_{t-1}$. Condition {\bf (H2)} holds on the scoring rule $S$, the parametrization $x\mapsto P_x$ and the distribution $P_t$ of the variable of interest $y_t$ given ${\cal F}_{t-1}$
\[
S(P_{y},P_{t})\le S(P_x,P_t)+\nabla_y S(P_{y},P_{t})^T(y-x)-\frac\alpha2 \E_{t-1}[(\nabla_y S(P_{y},y_{t})^T(y-x))^2]\,,  x,y\in\mathcal K\,.
\]
If Condition {\bf (H2)} is satisfied in  well-specified settings $P_t=P_{x_t^*}$, for any $x_t^*\in {\cal K}$, then $S$ is a proper scoring rule for the class $\{P_x; x\in {\cal K}\}$ in the sense of  \cite{gneiting2007strictly}; $S(P_{y},P_{t})$ is minimum when $P_{y}=P_t$ by convexity. The scoring rule is not necessarily strictly proper since this maximum is not unique when $\nabla_y S(P_{y},P_{t})$ is null in some directions $y$ in the neighborhood of $x_t^*$. 

We provide examples of time series probabilistic forecasting calibrated using the SOCO analysis by verifying Condition {\bf (H2)}. We focus on the logarithmic score assuming that $P_x$, $P_t$ admit densities $p_x$, $p_t$, $x\in {\cal K}$, $t\ge1$. We have
\[
L_t(x_t)=S(P_{x_t},P_{t})=-\E_{t-1}[\log(p_{x_t}(y_t))]=KL(P_t,P_{x_t})-\E_{t-1}[\log(p_t(y_t))]
\]
where $KL$ is the Kullback-Leibler divergence. This scoring rule is strictly proper because $S$ is minimized when $P_{y}=P_t$ only.
It is likely to satisfy the stochastic exp-concavity condition {\bf (H2)} locally in well-specified settings.
\begin{proposition}
If $P_t$ is in the exponential family so that its conditional density $p_t(y)$ is proportional to $e^{T(y)^Tx^*_t-\ell_t(x^*_t)} $ with sufficient statistic $T(y)$ and some $x^*_t\in{\cal K}$ then for the logarithmic score
\begin{multline*}
\E_{t-1}[\nabla \ell_t(x^*_t)\nabla \ell_t(x^*_t)^T]=\E_{t-1}[\nabla_{x^*_t} S(P_{x^*_t},y_{t})^T\nabla_{x^*_t} S(P_{x^*_t},y_{t})^T]\\=\nabla_{x^*_t}^2 S(P_{x^*_t},P_{x^*_t})=\nabla^2 L_t(x^*_t)\,,
\end{multline*}
and necessarily $\alpha\le 1$ if condition {\bf (H2)} holds. 
\end{proposition}
\begin{proof}
We apply Proposition \ref{prop:alpha}, noticing that the Fisher information identity holds in the well-specified setting.
\end{proof}
We use the logarithmic score for calibrating the first and second moments of gaussian forecasters as recommended in Section 4.4 of \cite{gneiting2007strictly}.  \cite{giraud2015aggregation} focus on the estimation of $m_t=\E_{t-1}[y_t]$, establishing fast-rate stochastic regret bounds in expectation. 
\begin{example}[Estimation of the conditional expectation]\label{ex:normal}
Let $P_x={\cal N}(x, \sigma^2)$ so that $\ell_t(x)=(x-y_t)^2/(2\sigma^2)$ (plus constant). In the OCO analysis, $\ell_t$ is $ \sigma^2/D^2$-exp-concave only if $y_t\in {\cal K}$ satisfying {\bf (H1)}. This setting requires implicitly that the distributions $P_t$ are $\cal K$ supported. Unbounded cases $y_t\notin\cal K$ are analyzed by SOCO assuming that the conditional distribution $P_t$ has mean $m_t=\E_{t-1}[y_t]\in{\cal K}$ and finite conditional variance $\sigma_t^2=\v_{t-1}(y_t)\le \overline \sigma^2$ a.s., for some $\overline \sigma^2>0$ and all $t\ge1$. The losses $\ell_t$ are not exp-concave but still satisfy Condition {\bf (H2)} with $\alpha= \sigma^2/(\overline \sigma^2+D^2)$; See Proposition \ref{prop:mean} for more details. The well-specified unbounded case $P_t={\cal N}(x,\sigma_t^2)$ satisfied Condition ({\bf H2}) with $\alpha= \sigma^2/(\overline \sigma^2+D^2)$ when $m_t\in {\cal K}$ and $\sigma_t^2\le \overline \sigma^2$.
\end{example}
We also focus on the estimation of the conditional variance or volatility $\sigma_t^2=\v_{t-1}(y_t)$ for gaussian probabilistic forecasters. Up to our knowledge, stochastic regret bounds for sequential algorithms calibrating the volatility have not been established yet. However, the concept of volatility is important and required in many applications such as risk assessment and probabilistic forecasting in finance \citep{mcneil2015quantitative,shafer2019game}.  The logarithmic score is well-suited to measure the performances of volatility estimators as it is robust to extreme values \citep{patton2011volatility}. 
\begin{example}[Estimation of the volatility] Let $P_x={\cal N}(m_t,x)$ then $\ell_t(x)=(\log(x)+(y_t-m_t)^2/x)/2$ (plus constant) is convex only if $0<x\le 2(y_t-m_t)^2$. This assumption is unrealistic  when $y_t$ is concentrated around its conditional mean $m_t$. Using SOCO, if  the conditional distribution $P_t$ has mean $m_t$ and volatility $\sigma_t^2\in \mathcal K=[c\overline{\sigma}^2/2,\overline{\sigma}^2]$, $ \overline \sigma^2>0$, $1<c<2$, then the risk $L_t(x)=(\log(x)+\sigma_t^2/x)/2$ is strongly convex with $\mu=(c-1)/(2\overline{\sigma}^4)$. Condition {\bf (H2)} is satisfied with $\alpha=(c-1)c^42^{-6}$ if $\E_{t-1}[(y_t^2-\sigma_t^2)^2]\le 3\overline{\sigma}^4$; See Proposition \ref{prop:vol} for more details. 
\end{example}
The stochastic exp-concavity condition is well-preserved for linear multivariate parametrization. Thus the conditional expectation and the volatility can be expressed as a linear combination of the past observations $y_{t-1},\ldots,y_1$ or their squares $y_{t-1}^2,\ldots, y_1^2$. We obtain naturally AR and ARCH estimations for the conditional expectation and the volatility in Sections \ref{sec:arma} and  \ref{sec:volfor}, respectively. Combining both, we obtain the AR-ARCH gaussian forecaster studied in Section \ref{sec:pred}. The parametrization does not preserve the strictly proper property of the logarithmic loss function. Despite the logarithmic score being strictly proper overall probability measures, it is not for the AR-ARCH models because different linear combinations of past observations provide the same probabilistic forecaster. Stochastic exp-concavity condition {\bf (H2)}, more general than strict properness, is crucial.


\subsection{Sequential ARMA prediction by BOA-ONS}\label{sec:arma}

AutoRegressive Moving Average (ARMA) modeling of the conditional mean is standard in time series analysis. See \cite{brockwell2009time} for a reference textbook. We calibrate sequentially, using the SOCO analysis with the natural filtration ${\cal F}_t=\sigma(y_t,\ldots, y_1)$, and the gaussian forecaster $\mathcal N(\widehat m_t^{(p)}(x),\sigma^2)$ for arbitrary $\sigma^2>0$, with clipped mean
\[
\widehat m_t^{(p)}(x)= x^T((y_{t-1}\wedge M/2)\vee (-M/ 2),\ldots,(y_{t-p}\wedge M/ 2)\vee (-M/ 2)),\,M>0\,,
\]
and ${\cal K}=B_1(1)$, the $\ell^1$ unit-ball of dimension $p$. 
\begin{proposition}\label{prop:mean}
We assume that the distributions $P_t$ of $y_t$ given $y_{t-1},\ldots,y_1$, admit densities  with means $m_t$ satisfying $2|m_t|\le M$, volatilities $\sigma_t^2\le \overline \sigma^2$ a.s.,  $M>0$, $\overline \sigma^2>0$, for every $t\ge1$, and satisfy {\bf(H3)}. Then the gaussian forecaster $\mathcal N( m_t^{(p)}(x),\sigma^2)$ calibrated by the ONS algorithm with $\gamma=\sigma^2/(3(\overline \sigma^2+M^2))$ achieves the stochastic regret
\begin{multline*}
\sum_{t=1}^TKL_t(P_t,\mathcal N(\widehat m_t^{(p)}(x_t),\sigma^2))- \sum_{t=1}^TKL(P_t,\mathcal N( m_t^{(p)}(x),\sigma^2))\\\le O\Big(\frac{\overline\sigma^2+M^2}{\sigma^2}p \log T + \Big(\frac{\overline\sigma^2+M^2}{\sigma^2}+\frac{\sigma^2}{\overline\sigma^2+M^2} p\, G_{\psi_2}^2\Big)\log(\delta^{-1}))\Big)\,,
\end{multline*}
for every $ T\ge 1$, $x\in B_1(1)$, and with high probability $1-\delta$.
\end{proposition}
To tackle the case of ARMA models with a moving average component, we consider increasing orders $p$ since any invertible ARMA model admits an AR($\infty$) representation. For the orders $p\in\{1,\ldots,\sqrt{\log T}/\log\log T\}$, the ONS predictors $\widehat m_t^{(p)}(x_t)$ are aggregated with BOA in $\widehat m_t$. The obtained BOA-ONS algorithm achieves the cumulative $KL$-divergence bound  
\begin{multline}
\sum_{t=1}^TKL(P_t,\mathcal N(\widehat m_t,\sigma^2))\le \min_{1\le p\le \sqrt{\log T}/\log\log T}\min_{x\in B_1(1)}\Big\{\sum_{t=1}^TKL(P_t,\mathcal N(m_t^{(p)}(x),\sigma^2))\\
+ O\Big(\frac{\overline \sigma^2+M^2}{\sigma^2}p\log T\Big)+O\Big(\frac{\overline \sigma^2+M^2}{\sigma^2}+\frac{\sigma^2}{\overline \sigma^2+M^2}p\, G_{\psi_2}^2\Big)\log(\delta^{-1})\Big)\Big\}\,,\label{eq:regretklarma}
\end{multline}
refining the bound  obtained by \cite{anava2013online}. Our bound is valid in every sub-gaussian stochastic adversarial settings where $2|m_t|\le D$, and the time series $(y_t)$ does not have to be bounded as in \cite{anava2013online}. Moreover, our bounds are any-time valid with high probability.

The parameters $(M,\sigma^2)$ should be tuned to find the best compromise in the regret bound \eqref{eq:regretklarma}. However, the task is not feasible using the SOA analysis because the loss functions depend on these parameters. The solution comes from the econometrics litterature that provides better loss and risk functions introducing the concept of volatility.

\subsection{Sequential ARCH prediction by BOA-ONS}\label{sec:volfor}

In mathematical finance, the log-ratios $(y_t)$ are commonly modeled  using Generalized AutoRegressive Conditionally Heteroscedastic (GARCH) model. Classical inference uses the Quasi-Likelihood approach \citep{francq2019garch} as if the conditional distributions were gaussian. If the conditional means $m_t:=\E[y_t\mid y_{t-1},\ldots,y_1]$ are null, the volatilities $\sigma^2_t:=\v(y_t\mid y_{t-1},\ldots,y_1)$ are finite, $t\ge 1$, then the Quasi-Likelihood estimator $\widehat\sigma^2_t$ of the volatility minimizes the cumulative KL divergence $KL(P_t,\mathcal N(0,\widehat\sigma_t^2))=(\log(2\pi\widehat\sigma_t^2)+\sigma^2_t/\widehat\sigma_t^2)/2$ (plus constant).

We assume that  $\sigma_t^2\in[c\overline\sigma^2/3,\overline \sigma^2]$, $1<c<2$, $\overline\sigma^2>0$, and we use a clipped-ARCH(q) model  
\begin{equation}\label{eq:arch}
\widehat\sigma_t^{2,(q)}(x)=c\overline\sigma^2/2+x_1 (y_{t-1}^2\wedge \overline \sigma^2)+\cdots + x_q(y_{t-q}^2\wedge \overline \sigma^2)\,,
\end{equation}
with $x\in {\cal K}=\{x\in \R^{q}:\, x\ge {\bf 0}\text{ and } \|x\|_1\le 1-c/2 \}$.

\begin{proposition}\label{prop:vol} We assume that the distributions $P_t$ of $y_t$ given $y_{t-1},\ldots,y_1$, admit densities with means $m_t=0$, volatilities $\sigma^2_t\in [c\overline\sigma^2/2, \overline\sigma^2]$,  $\E_{t-1}[(y_t^2-\sigma_t)^2]\le 3\overline\sigma^4$, a.s., $1<c<2$, $\overline \sigma^2>0$ for every $t\ge1$, and satisfy {\bf(H3)}. Then the gaussian forecaster $\mathcal N(0,\widehat\sigma_t^{2,(q)}(x))$ calibrated by the ONS algorithm with $\gamma=2^6/(3(c-1)c^4)$ achieves the stochastic regret
\[
\sum_{t=1}^TKL(P_t,\mathcal N(0,\widehat\sigma_t^{2,(q)}(x_t)))- \sum_{t=1}^TKL(P_t,\mathcal N(0,\widehat\sigma_t^{2,(q)}(x)))\le O(q (\log T +G_{\psi_2}^2 \log(\delta^{-1})))\,,
\]
for every $ T\ge 1$, $x\in {\cal K}$ with high probability $1-\delta$.
\end{proposition}
Any invertible GARCH model admits an ARCH($\infty$) representation. Thus we consider ARCH($q$) models with increasing order $q$. We consider BOA-ONS $\widehat \sigma_t^2$ aggregating $\widehat \sigma_t^{2,(q)}(x_t)$, $q=1,\ldots,\sqrt{\log T}/\log\log T$ so that with high probability \begin{multline*}
\sum_{t=1}^TKL(P_t,\mathcal N(0,\widehat\sigma_t^{2}))\le  \min_{1\le q\le \sqrt{\log T}/\log\log T}\min_{x\in \mathcal K}\Big\{\sum_{t=1}^TKL(P_t,\mathcal N(0,\sigma^{2,(q)}(x)))
\\
++ O(q(\log T+G_{\psi_2}^2  \log(\delta^{-1})))\Big\}\,.
\end{multline*}
We solve positively the question raised in the conclusion of
\cite{anava2013online} about the optimization of GARCH forecasters. The main restriction of our approach is the small range of the volatilities $[c\overline\sigma^2/2,\overline \sigma^2]$, $1<c<2$. Otherwise, the risk functions arer not even convex when the volatility $\sigma_t^2$ can be over-estimated by a factor of 2. It is not surprising since \cite{francq2010inconsistency} showed that the Quasi-Likelihood approach is inconsistent with no  lower boundedness assumption on the volatilities. 

\subsection{Online gaussian probabilistic forecasting using BOA-ONS}\label{sec:pred}

We combine the ARMA and volatility prediction methods. We consider gaussian probabilistic forecaster $\mathcal N(\widehat m_t^{(p)}(x_{1:p}),\widehat\sigma_t^{2,(q)}(x_{p+1:p+q}))$ with $M^2=\overline \sigma^2$ and 
\[
x=(x_{1:p},x_{p+1:p+q})\in {\cal K}=\{x\in \R^{p+q}:\, \|x_{1:p}\|_1\le 1,\;x_{p+1:p+q}\ge {\bf 0}, \|x_{p+1:p+q}\|_1\le 1-c/2\}.
\] 
\begin{proposition} We assume that the distributions $P_t$ of $y_t$ given $y_{t-1},\ldots,y_1$, admit densities with means $2|m_t|\le \overline \sigma$, volatilities $\sigma_t^2\in [c\overline \sigma^2/2, \overline \sigma^2]$,  $\E_{t-1}[(y_t-m_t)^4]\le 3\overline\sigma^4$, a.s., $1<c<2$, $\overline \sigma^2>0$ for every $t\ge1$, and satisfy {\bf(H3)}. Then the gaussian forecaster $\mathcal N(\widehat m_t^{(p)}(x),\widehat\sigma_t^{2,(q)}(x))$ calibrated by the ONS algorithm with $\gamma=3\times2^5/((c-1)c^4)$ achieves the stochastic regret
\begin{multline*}
\sum_{t=1}^TKL(P_t,\mathcal N(\widehat m_t^{(p)}(x_{1:p}),\widehat\sigma_t^{2,(q)}(x_{p+1:p+q})))- \sum_{t=1}^TKL(P_t,\mathcal N(\widehat m_t^{(p)}(x_{1:p}),\widehat\sigma_t^{2,(q)}(x_{p+1:p+q})))\\
 \le O((p+q)( \log T +G_{\psi_2}^2) \log(\delta^{-1})))\,,
\end{multline*}
for every $ T\ge 1$, $x\in {\cal K}$ with high probability $1-\delta$.
\end{proposition}

Aggregating such predictors for $1\le p,q \le \sqrt{\log T}/\log\log T$ with BOA, we obtain a gaussian probabilistic forecast $\mathcal N(\widehat m_t,\widehat\sigma_t^2)$ satisfying the cumulative $KL$-divergence bound
\begin{align*}
\sum_{t=1}^TKL(P_t,\mathcal N(\widehat m_t,\widehat\sigma_t^{2}))\le & \min_{1\le q,p\le \sqrt{\log T}/\log\log T}\min_{x\in \mathcal K}\Big\{\sum_{t=1}^TKL(P_t,\mathcal N(\widehat m_t^{(p)}(x_{1:p}),\widehat\sigma_t^{2,(q)}(x_{p+1:p+q})))\\
&+ O((p+q)(\log T+ G_{\psi_2}^2\log(\delta^{-1})))\Big\}
\end{align*}
with high probability.  The sequential algorithms adapt to the random environment even in misspecified settings; It approximates the parametric gaussian forecaster that is the closest to the unknown conditional distributions for the cumulative KL divergences and a penalty which increases such as $(p+q)\log(T)$. Thus the BOA-ONS forecaster regret minimizes automatically a Bayesian information type criterion at any-time and with high probability. It is comparable to a model selection procedure that would require to minimize a penalized log-likelihood at each step $1\le t\le T$. The computational cost of our recursive method is $O(T((p+q)^2+P))$ with explicit formulae except for the projection step of computational cost $P$, whereas the batch model selection has a computational cost $O(T(p+q)M)$ where $M$ is the computational cost of the optimization of the likelihood in AR($p$)-ARCH($q$) models. This cost $M$ is prohibitive when $p+q$ is large and the computational gain of our recursive procedure is important.  
%
%
%
\subsection{Sequential probabilistic forecasting using BOA-ONS}

The main drawback of our BOA-ONS approach on gaussian forecasters is the restriction $\sigma_t^2\in[c\overline\sigma^2/2,\overline \sigma^2]$, $1\le t\le T$. However, because the loss and risk functions depend on this hyperparameter, it is not possible to directly aggregate volatility estimators with different $\overline\sigma^2>0$ in a gaussian forecaster to extend the range of the volatilities.\\

To circumvent the issue, we can aggregate the gaussian probabilistic forecasters to obtain a probabilistic forecaster which is mixed gaussian. Consider $\widehat P_t=(\widehat P^{(i)}_t)_{1\le i\le K}$, $K$ weak probabilistic forecasters with densities $\widehat p_t=(\widehat p_t^{(i)})_{1\le i\le K}$ such as $\widehat P^{(i)}_t=\mathcal N(\widehat m_t^{(j)},\widehat\sigma_t^{2,\ell})$ with different localization, $jD+D/\sqrt{2} \le \widehat m_t^{(j)}\le (j+1)D+D/\sqrt{2}$, for $- K_1\le j\le K_2$ and $\widehat\sigma_t^{2,\ell}\in[(c/2)^{\ell+1}\overline\sigma^2/2,(c/2)^\ell \overline\sigma^2]$ for $0\le \ell\le K_3$. Consider the SOCO analysis of mixtures $x^T\widehat P$ with $\mathcal K= \Lambda_K$ and $\ell_t(x)=-\log(x^T\widehat p_t(y_t))$. We assume that $m\le \E_{t-1}[1/\widehat p_t^{(i)}(y_t)^2]\le M$ a.s. for $1\le i\le K$, $t\ge 1$.  The risk function is $m$-strongly convex and Condition {\bf (H2)} is satisfied with $\alpha=m/M$. 

Under Condition {\bf (H3)}, we can use the ONS algorithm on the simplex ${\cal K}=\Lambda_K$, and we obtain 
\begin{align*}
\sum_{t=1}^TKL(P_t,\pi_t^T\widehat p )\le \min_{\pi\in\Lambda_K} \sum_{t=1}^TKL(P_t,\pi^T\widehat p)+O(MK/m( \log T+\log(\delta^{-1})))\,.
\end{align*}
Similar fast-rate regret bounds were obtained by \cite{thorey2017online} for the CRPS score instead of the KL divergence. They used the Recursive Least Square algorithm without projection that does not constrain $\pi_t$ to be in $\Lambda_K$. Contrary to our procedure, it is difficult to interpret their ensemble probabilistic forecast because they do not satisfy the axioms of a density function. 

\section{Numerical illustrations}\label{sec:illustr}

\subsection{Aggregations in stochastic environments}

\begin{figure}[ht]
\centerline{
\includegraphics[scale=.25]{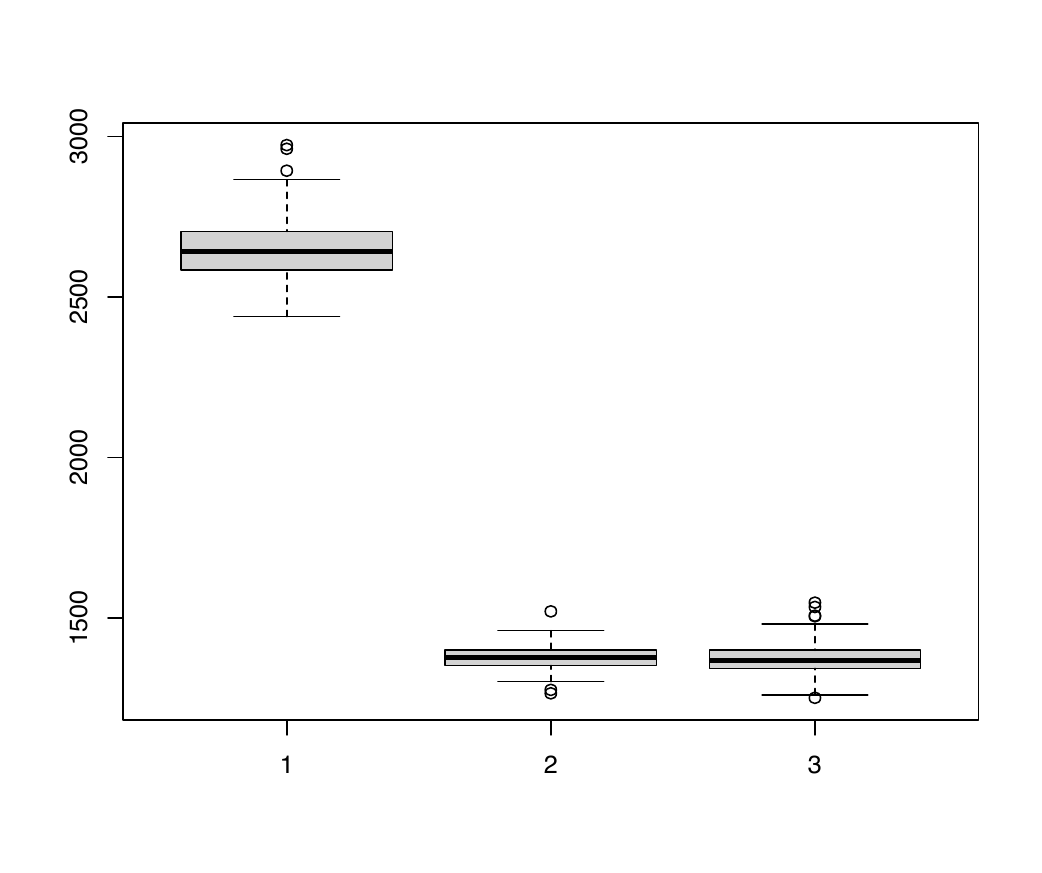}\quad\includegraphics[scale=.25]{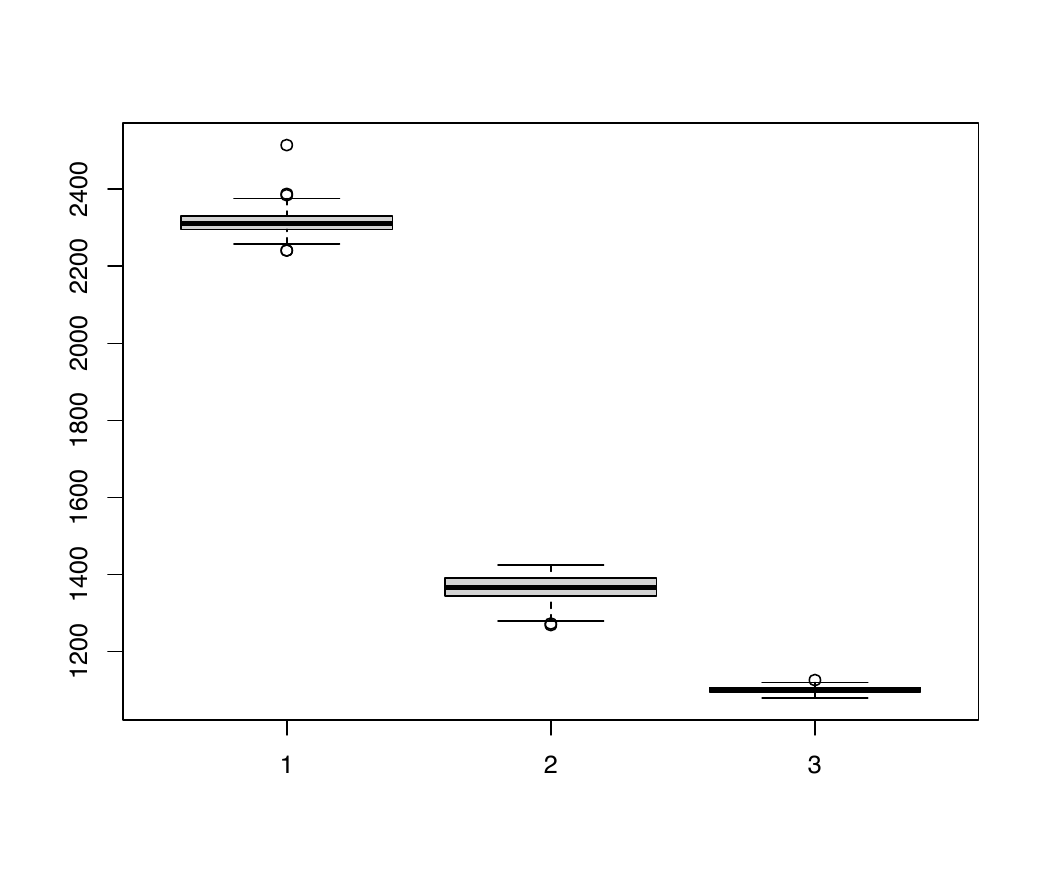}\quad\includegraphics[scale=.25]{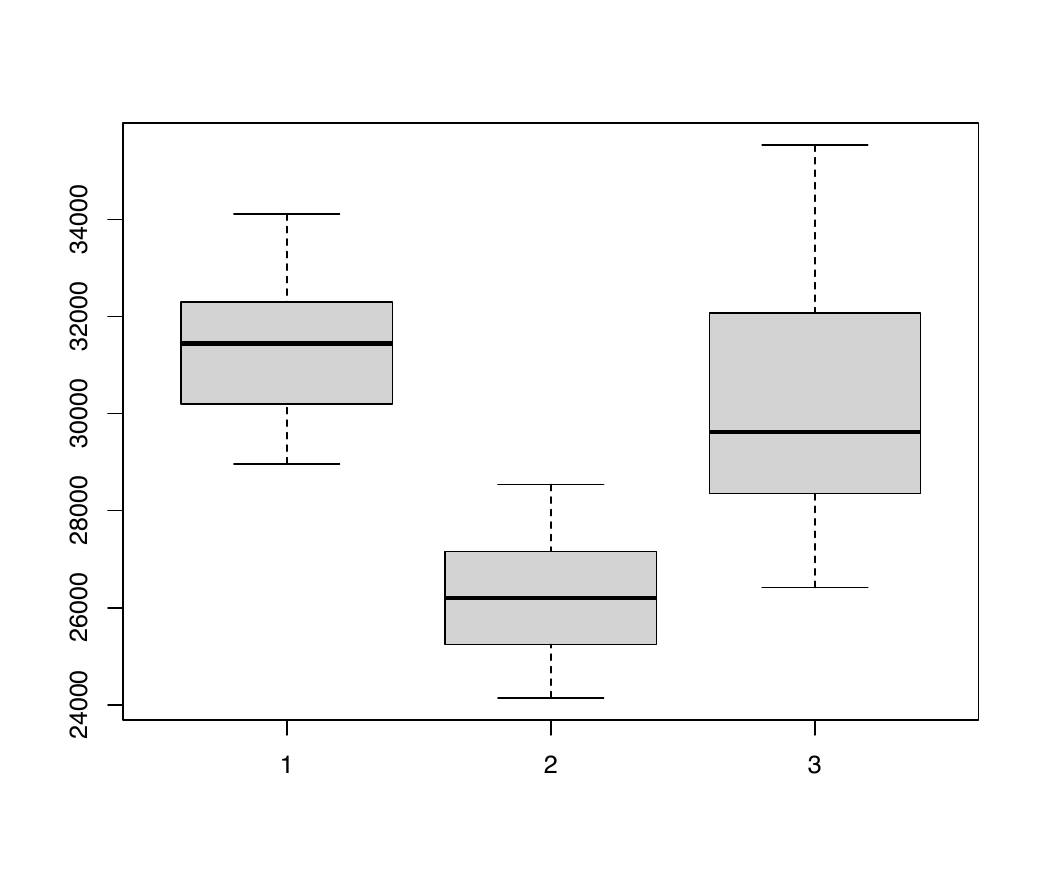}}
\caption{Boxplots of original BOA (1), scale-free version of BOA (2), and squint (3) for $\sigma=1$ (left), $\sigma=.1$ (middle) and $\sigma =10$ (right).}\label{fig:bp}
\end{figure}

We study the impact of stochastic deviations on the aggregation of predictors for quadratic losses. We consider $100$ predictors of $y_t=0$, $t\ge 1$, the first one being negatively biased $-\sqrt t + \sigma N_t^{(1)}$, the other ones being positively biased $\sqrt t + \sigma N_t^{(i)}$, $1\le t \le 1000$, $2\le i\le 100$. Here $N_t^{(i)}$ are iid standard gaussian random variables. Any aggregation half-weighting the first predictor does not suffer from the bias. We run 100 Monte-Carlo experiments of three different aggregation algorithms; The original version of BOA of \cite{wintenberger2017optimal}\footnote{The multiple tuning of the deviation bounds in the original version of BOA is flawed and replaced by the univariate doubling trick of \cite{cesa2007improved}.}, the scale-free version of BOA of Algorithm \ref{alg:boa}, and the squint algorithm of \cite{koolen2016combining}. The latter algorithm is not comparable as it uses beforehand the maximum of the deviations for initializing the algorithm. As expected, its performances compared with the scale-free version of BOA highly depend on the level of the stochastic deviations, outperforming it when $\sigma=.1$; See Figure \ref{fig:bp}. The original version of BOA does not manage to learn efficiently the range of small deviations and does not outperform the scale-free version of BOA in this case because it also suffer from a small observed minimal loss $\underline{m}$. On the opposite, for large deviations, both BOA versions achieve performances competitive with squint, the scale-freee BOA outperforming the two other algorithms when $\sigma=10$.

\subsection{Quantile prediction of electricity loads}\label{sec:predint}

\begin{figure}[ht]
\centerline{
\includegraphics[scale=.25]{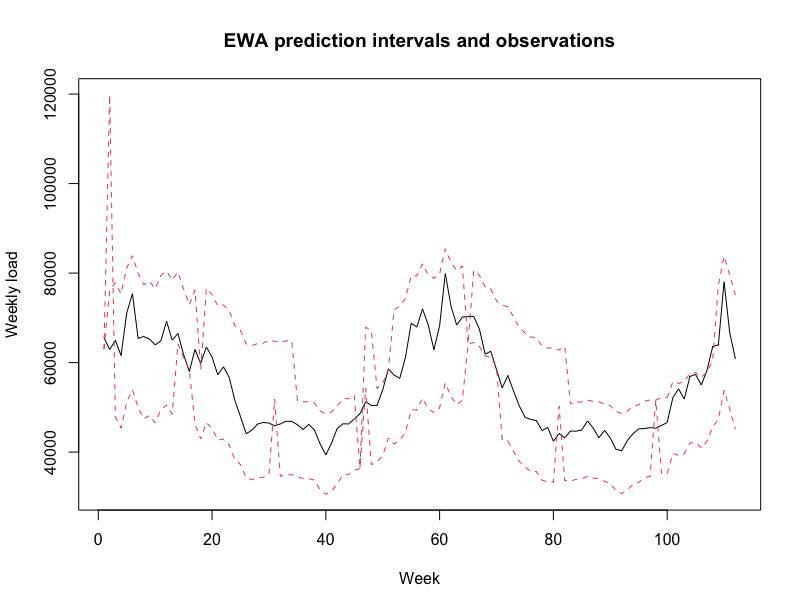}\qquad\includegraphics[scale=.25]{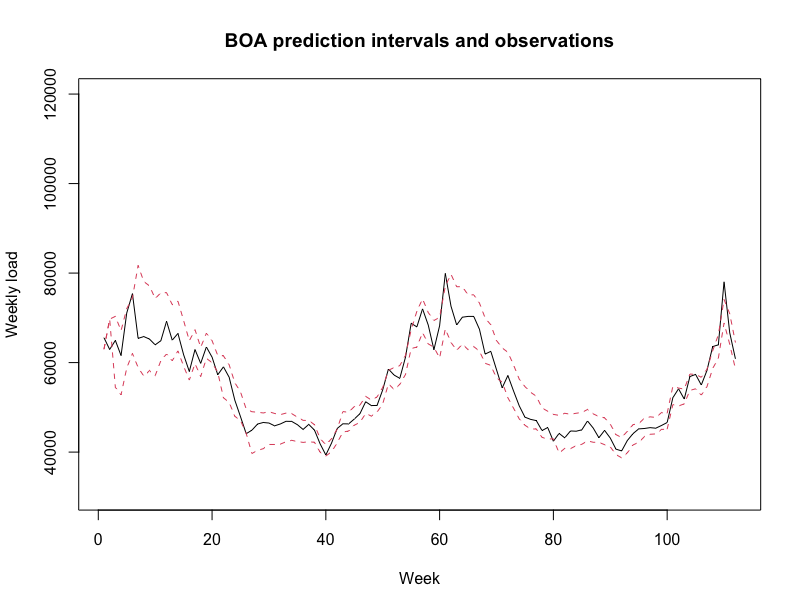}}
\caption{$90\%$-prediction intervals of the electricity load based on EWA (left) and BOA (right) and the same 5 forecasters.}
\end{figure}
We illustrate the impact of the SOCO anlysis on quantile predictions for weekly electricity load, data available in the Opera package developed by  \cite{opera}. The 3 forecasters (GAM, AR, GBM) provided in Opera package plus 2 constant forecasters, 0 and 1.5 times the maximum of weekly loads, are aggregated to predict the upper and lower quantile of levels .5 and .95. We use the quantile  loss funcion in 2 different sequential aggregation algorithms, Exponentially Weighted Algorithm (EWA) and BOA, and for the two levels .5 and .95. BOA aggregations provide accurate quantile predictions because it minimizes cumulative risks in the SOA analysis. It confirms the theoretical guarantees obtained in the paper since it is likely that the pinball risk is strongly convex   \citep{steinwart2011estimating}. On the contrary EWA aggregations  fail to provide accurate quantile predictions because EWA algorithm minimizes the cumulative losses which are not exp-concave. Such visual validation of the predictions interval is enough to show the benefit of BOA but does not constitute any evidence of its good calibration.   \cite{biau2011sequential} analyze the asymptotic guarantees of a different sequential algorithm predicting quantiles.

\subsection{Volatility estimation during the COVID crisis}
We apply BOA-ONS for designing $90\%$-prediction intervals for the S\&P500 index during 2020, including the COVID crisis in March. We use the iid $\mathcal N(0,x)$ and ARCH($p$) gaussian probabilistic forecasters for $p=1,\ldots,5$. The forecasters are tuned sequentially with the ONS algorithm with $\gamma=1$, and  ${\cal K}=[c,\infty]\times B_1(1)$, $c=0$ in the iid case, and $c=10^{-16}$ in the ARCH cases. These 6 predictors of the volatility are then aggregated with BOA; See Figure \ref{fig:sp500}. We notice that the iid forecast prediction interval is constant after some training period. The ARCH forecasts are required to predict intervals accurately during the crisis.  BOA aggregations converge to weights $(0.01,0.17, 0.09, 0.37, 0.21, 0.15)$ and improve the calibration of ARCH forecasters. A slightly more advanced sequentially calibrated volatility estimator  developed by \cite{werge2022adavol} has been used in the forecast task of the M6 financial competition by \cite{de2023adaptive}. Its RPS performances rank 5th out of 163 competitors, showing that such sequential calibration o  is competitive in probabilistic forecasting.
\begin{figure}[ht]\label{fig:sp500}
\centerline{
\includegraphics[height=6cm,width=7cm]{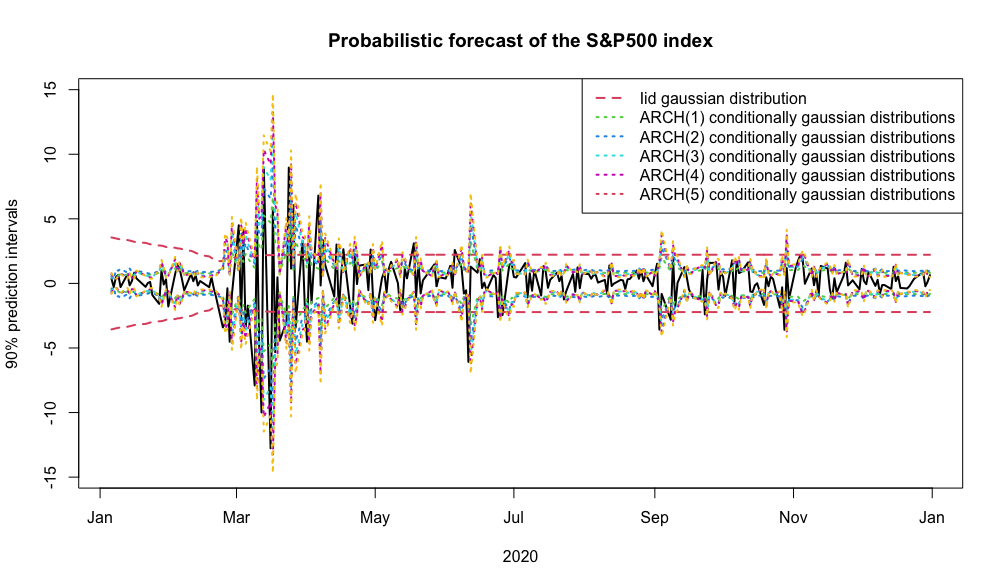}\qquad\includegraphics[height=6cm,width=7cm]{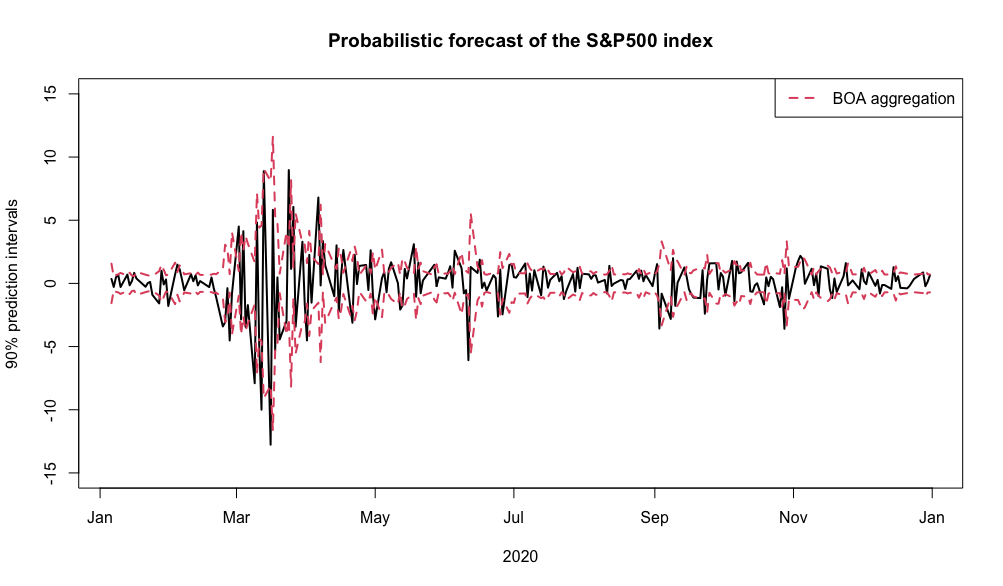}}
\caption{$90\%$-prediction intervals from 6 forecasters (left) and their BOA aggregation (right).}
\end{figure}

\section{Conclusion and future works}\label{sec:concl}

In this paper, we derive fast-rate stochastic regret bounds for the ONS and BOA algorithms under stochastic exp-concavity. We alleviate the convexity assumption on the loss functions to calibrate sequentially parametric probabilistic forecasting using the logarithmic score. We achieve fast-rate stochastic regret bounds. Thus, BOA-ONS can adaptively and efficiently calibrate gaussian probabilistic forecasters for any conditionally sub-gaussian non-stationary time series. Our stochastic regret bounds are relative to a static prediction parametrized by  $x\in {\cal K}$ for every $t\ge1$. When forecasting non-stationary time series, we should also consider competitors that evolve through time. Key Propositions \ref{prop:surrogate} and \ref{prop:pois} extend readily to such settings called tracking optimization problems. Thus, one would like to develop SOCO and algorithms in more dynamic settings. A first step in that direction is made in \cite{haddouche2023optimistic} using optimistic sequential  algorithms.
\vspace*{-10pt}

\bibliography{biblioSOCO}
\appendix

\section{Proofs of the main results}\label{sec:proofs}
\subsection{Proof of Proposition \ref{prop:gauss}}
We first show that 
\[
\Big\|\dfrac{\nabla \ell_t(x_t)^T(x_t-x)}{\sqrt{\E_{t-1}[(\nabla \ell_t(x_t)^T(x_t-x))^2]}}\Big\|_{\psi_2}\le \sqrt{8/3 + (1/\log 2)^2}=K_{\psi_2}\approx 2.179\,.
\]
Then we derive 
\[
\dfrac{\E_{t-1}[(\nabla \ell_t(x_t)^T(x_t-x))^{2k}]}{\E_{t-1}[K_{\psi_2}^2(\nabla \ell_t(x_t)^T(x_t-x))^2 ]^k}\le k!\E_{t-1}\Big[\psi_2\Big(\dfrac{\nabla \ell_t(x_t)^T(x_t-x)}{\sqrt{\E_{t-1}[K_{\psi_2}^2(\nabla \ell_t(x_t)^T(x_t-x))^2]}}\Big)\Big]\le 2k!\,.
\]
Using Cauchy-Svhwarz inequality we derive that $\E_{t-1}[(\nabla \ell_t(x_t)^T(x_t-x))^2 ]\le \E_{t-1}[\|\nabla \ell_t(x_t)\|^2 ]D^2\le G^2_2D^2$ and
\begin{align*}
\E_{t-1}[(\nabla \ell_t(x_t)^T(x_t-x))^{2k}]&\le k! 2K_{\psi_2}^{2k}\E_{t-1}[(\nabla \ell_t(x_t)^T(x_t-x))^2 ]^k\\
&\le k! 2K_{\psi_2}^{2k}(G_2D)^{2(k-1)}\E_{t-1}[(\nabla \ell_t(x_t)^T(x_t-x))^2 ]\,.
\end{align*}
Then we fix $G_{\psi_2}=2K_{\psi_2}^2G_2$ so that Condition {\bf (H3)} follows.
Let us denote $\mu_t$ and $\sigma_t$ the mean and the variance of the conditionally gaussian random variable. Then, $N$ being standard gaussian distributed, we use the homogeneity and triangular inequality on the norm $\|\cdot\|_{\psi_2}$ to derive
\begin{align*}
\Big\|\dfrac{\nabla \ell_t(x_t)^T(x_t-x)}{\sqrt{\E_{t-1}[(\nabla \ell_t(x_t)^T(x_t-x))^2]}}\Big\|_{\psi_2}&=\Big\|\dfrac{\sigma_t N + \mu_t}{\sqrt{\E_{t-1}[(\sigma_t N + \mu_t)^2]}}\Big\|_{\psi_2}\\
&\le\dfrac{\sigma_t \|N \|_{\psi_2}+ \|\mu_t\|_{\psi_2}}{\sqrt{\sigma_t^2 + \mu_t^2}}\\
&=\dfrac{\sigma_t\sqrt{8/3}+ \mu_t/\log 2}{\sqrt{\sigma_t^2 + \mu_t^2}}
\end{align*}
and the desired results follows from Cauchy-Schwartz inequality.

\subsection{Proof of Proposition \ref{prop:surrogate}}
Denoting $Y_t=\nabla \ell_t(x_t)^T(x_t-x)$, we observe that under {\bf (H2)} it holds
\begin{equation}\label{eq:h2bis}
\sum_{t=1}^TL_t(x_t)-\sum_{t=1}^TL_t(x)\le \sum_{t=1}^T\E_{t-1}[Y_t]-\frac\alpha2\E_{t-1}[Y_{t}^2].
\end{equation}
Moreover, from Lemma B.1 of \cite{bercu2008exponential} for any random variable $Y_t$ and any $\eta\in \R$ we have 
\[
\E_{t-1}\Big[\exp(\eta (Y_t - \E_{t-1}[Y_t])-\frac{\eta^2} 2(\E_{t-1}[Y_{t}^2]-\E_{t-1}[Y_{t}]^2+(Y_{t}-\E_{t-1}[Y_{t}])^2))\Big]\le 1\,.
\]
Developing the square, we obtain
\[
\E_{t-1}\Big[\exp(\eta (Y_t - \E_{t-1}[Y_t])-\frac{\eta^2} 2(\E_{t-1}[Y_{t}^2]+Y_{t}^2)+\eta^2\E_{t-1}[Y_t]Y_t\Big]\le 1\,.
\]
Using Young's inequality together with Jensen's one, we derive 
\[
\E_{t-1}[Y_t]Y_t\ge -(\E_{t-1}[Y_t]^2+Y_t^2)/2\ge -(\E_{t-1}[Y_t^2]+Y_t^2)/2
\]
and the exponential inequality
\[
\E_{t-1}[\exp(\eta (Y_t - \E_{t-1}[Y_t])-\eta^2(\E_{t-1}[Y_{t}^2]+Y_{t}^2)]\le 1\,.
\]
We obtain the desired result applying a classical martingale argument due to \cite{ville1939etude} and \cite{freedman1975tail} and recalled in Appendix \ref{app:freedman}. Indeed, using the notation of Appendix \ref{app:freedman} with $Z_t=\eta (Y_t - \E_{t-1}[Y_t])-\eta^2(\E_{t-1}[Y_{t}^2]+Y_{t}^2)$, we have 
\[
\P(\exists T\ge 1:\; M_T >\delta^{-1})\le \delta\,,\qquad 0<\delta<1\,,
\]
where $M_T=\exp(\sum_{t=1}^TZ_t)$. Considering $\eta=-\lambda/2$ for any $\lambda>0$, it holds with probability $1-\delta$ for any $T\ge 1$
\begin{multline*}
\sum_{t=1}^T\Big(-\dfrac{\lambda}2(Y_t-\E_{t-1}[Y_t])-\dfrac{\lambda^2}4(\E_{t-1}[Y_t^2]+Y_t^2)\Big) \le \log(\delta^{-1})\\
\Leftrightarrow\sum_{t=1}^T\E_{t-1}[Y_t]\le\sum_{t=1}^TY_t +\dfrac{\lambda}2(\E_{t-1}[Y_t^2]+Y_t^2)+\dfrac2\lambda \log(\delta^{-1})
\end{multline*}
which, combines with \eqref{eq:h2bis}, yields the desired result.

\subsection{Proof of Theorem \ref{th:ons}}
From the proof of the ONS regret bound in \cite{hazan2016introduction}, we obtain from the expression of the recursive steps (and not using the convexity of the loss)
\[
 \sum_{t=1}^T\nabla \ell_t(x_t)^T(x_t-x)\le \frac\gamma2\sum_{t=1}^T(\nabla \ell_t(x_t)^T(x_t-x))^2+\frac1{2\gamma}\log(\det(A_T)/\det(A_0))+\frac1{2\gamma}.
\]
Plugging this inequality into the previous bound we obtain
\begin{align*}
\sum_{t=1}^TL_t(x_t)-\sum_{t=1}^TL_t(x)\le& \frac{\lambda+\gamma}2\sum_{t=1}^T(\nabla \ell_t(x_t)^T(x_t-x))^2\\
&+\frac{\lambda-\alpha}2\sum_{t=1}^T\E_{t-1}[(\nabla \ell_t(x_t)^T(x_t-x))^2]\\
&+\frac1{2\gamma}\log(\det(A_T)/\det(A_0))+\frac1{2\gamma}+\frac2\lambda\log(\delta^{-1})\,.
\end{align*}
Then we apply the Poissonian exponential inequality from Proposition \ref{prop:pois} on the second-order terms. More precisely, denoting $0\le Y_t=(\nabla \ell_t(x_t)^T(x_t-x))^2/(2(G_{\psi_2}D)^2)$, we obtain
\begin{equation}\label{eq:supmart}
\E_{t-1}[\exp(Y_t-2\E_{t-1}[Y_{t}])]\le 1\,.
\end{equation}
Combined with the argument due to \cite{freedman1975tail} recalled in Appendix \ref{app:freedman} we derive
\begin{equation}\label{eq:poisson}
\P\Big(\exists T\ge 1: \sum_{t=1}^T Y_t-2\sum_{t=1}^T \E_{t-1}[Y_{t}]>\log(\delta^{-1})\Big)\le \delta\,,\qquad 0<\delta<1.
\end{equation}
Thus an union bound provides
\begin{align*}
\sum_{t=1}^TL_t(x_t)-\sum_{t=1}^TL_t(x)\le
&\frac{3\lambda+2\gamma-\alpha}2\sum_{t=1}^T\E_{t-1}[(\nabla \ell_t(x_t)^T(x_t-x))^2]\\
&+\frac1{2\gamma}\log(\det(A_T)/\det(A_0))+\frac1{2\gamma}+((\lambda+\gamma)(G_{\psi_2}D)^2+\frac2\lambda)\log(\delta^{-1})\,.
\end{align*}
Choosing $3\lambda=\alpha-2\gamma>0$ since $\gamma<\alpha/2$ we conclude
\begin{equation}\label{eq:ons}
\sum_{t=1}^TL_t(x_t)-\sum_{t=1}^TL_t(x)\le\frac1{2\gamma}\log(\det(A_T)/\det(A_0))+\frac1{2\gamma}+\Big(\dfrac{\alpha+\gamma}3(G_{\psi_2}D)^2+\frac {6}{\alpha-2\gamma}\Big)\log(\delta^{-1})\,.
\end{equation}
From the initialization $A_0=\dfrac1{(\gamma D)^2}I_d$, we obtain bound 
\[
\log(\det(A_T)/\det(A_0))\le d\log\Big(1+(\gamma D)^2\sum_{t=1}^T\|\nabla \ell_t(x_t)\|^2\Big)\,.
\]
We apply the Poissonian exponential inequality from Propostion \ref{prop:pois} on the second-order terms $0\le Y_t=\|\nabla \ell_t(x_t)\|^2/(2G_{\psi_2}^2)$ and, combined with the argument due to \cite{freedman1975tail} and Condition {\bf (H3)} ensuring $\E_{t-1}[Y_t]\le G^2_2/(2G_{\psi_2}^2)$, we obtain
\[
\P\Big(\exists T\ge 1: \sum_{t=1}^T Y_t-T G^2_2/G_{\psi_2}^2>\log(\delta^{-1})\Big)\le \delta\,,\qquad 0<\delta<1.
\]
We derive that, with probability $1-\delta$, it holds
\[
\log(\det(A_T)/\det(A_0))\le d\log\big(1+2(\gamma D)^2(TG_2^2+G_{\psi_2}\log(\delta^{-1}))\big)\,,\qquad T\ge 1\,.
\]
The desired result follows from the specific choice of $\gamma$ and a union bound.
\subsection{Proof of Theorem \ref{thm:boa}}
We keep the same notation and convention as in Section \ref{sec:boa}. In particular, inequalities involving vectors are coordinatewise. With no loss of generality we assume that $\eta_{1,i}\neq 0$ for all $1\le i\le K$. 
To prove the regret bound \eqref{eq:boa} we will show that 
\begin{equation}\label{eq:pi0}
\pi_1^T\exp(-\eta_{T}\widetilde L_{T})\le \underbrace{\exp\Big( \sum_{i=1}^K1\{\max_{2\le t\le T}x_{t,i}>1/4\big\} \log(1+(\eta_{1,i}M_{T,i})^2)\Big)\Big(\ex+\frac12\pi_1^T\log\Big(\1+(\eta_{1 }M_{T })^2T\Big)\Big)}_{=:A_T}\,.
\end{equation}
From \eqref{eq:pi0} we derive  
\[
-\eta_{T}\widetilde L_{T}=\eta_{T}\Big(\sum_{t=1}^T(\pi_t^T\boldsymbol\ell_t\1-\boldsymbol\ell_{t})-\sum_{t=1}^T\eta_{t-1}(\pi_t^T\boldsymbol\ell_t\1-\boldsymbol\ell_{t})^2\Big)\le \log(\pi_1^{-1} A_T)
\]
so that 
\[
\sum_{t=1}^T\pi_t^T\boldsymbol\ell_t\1\le \sum_{t=1}^T\boldsymbol\ell_{t}+\sum_{t=1}^T\eta_{t-1}(\pi_t^T\boldsymbol\ell_t\1-\boldsymbol\ell_{t})^2+\dfrac{\log(\pi_1^{-1})}{\eta_{T}}+\dfrac{\log(A_T)}{\eta_{T}}\,.
\]
Since $\eta_{t}^{-2}=\eta_{t-1}^{-2}+2.2(\boldsymbol\ell_t-\pi_t^T\boldsymbol\ell_t\1)^2$ we obtain by rearranging the sum
\[
\sum_{t=1}^T\eta_{t-1}(\pi_t^T\boldsymbol\ell_t\1-\boldsymbol\ell_{t})^2=\frac1{2.2}\sum_{t=1}^T \eta_{t-1}(\eta_{t}^{-2}-\eta_{t-1}^{-2})\le\frac1{2.2}\Big(\sum_{t=1}^T\dfrac{\eta_{t-1}-\eta_{t}}{\eta_{t}^2}+\dfrac{1}{\eta_{T}}\Big).
\]
Thus we derive from a comparison sum-integral
\begin{align*}
\sum_{t=1}^T\dfrac{\eta_{t-1}-\eta_{t}}{\eta_{t}^2} \le \dfrac1{\eta_{T}}\qquad\Longrightarrow\qquad
\sum_{t=1}^T\eta_{t-1}(\pi_t^T\boldsymbol\ell_t\1-\boldsymbol\ell_{t})^2&\le\dfrac{1}{1.1\eta_{T}}\,.
\end{align*}
The learning rate satisfying the relation
\[
\dfrac{1/1.1+\log(\pi_1^{-1})+\log(A_T)}{\eta_{T}}\le(1/1.1+\log(\pi_1^{-1})+\log(A_T))\sqrt{2.2\sum_{t=1}^T(\pi_t^T\boldsymbol\ell_t\1-\boldsymbol\ell_{t})^2}\,,
\]
and the regret bound \eqref{eq:boa} follows from the expression of
$\log(A_T)$.

It remains to prove the exponential inequality \eqref{eq:pi0}. We use the identity
\[
\exp(-\eta_{T}\widetilde L_{T})= \exp(\eta_{T} (\pi_{T}^T\boldsymbol\ell_{T}\1-\boldsymbol\ell_{T})-\eta_{T}^2(\boldsymbol\ell_T-\pi_T^T\boldsymbol\ell_T\1)^2)\exp(-\eta_{T}\widetilde L_{T-1})\,.
\]
To initiate the recursion, we use the basic inequality $x\le x^\alpha +e^{-1}(\alpha-1)/\alpha$ for $x\ge 0$ and $\alpha\ge 1$ with $x=\exp(-\eta_{T}\widetilde L_{T-1})$ and $\alpha=\eta_{T-1}/\eta_T$ so that
\[
\exp(-\eta_{T}\widetilde L_{T-1})\le \exp(-\eta_{T-1}\widetilde L_{T-1})+e^{-1}\dfrac{\eta_{T-1}-\eta_T}{\eta_{T-1}}\,.\\
\]
We obtain 
\[
\exp(-\eta_{T}\widetilde L_{T})\le \exp(\eta_{T} (\pi_{T}^T\boldsymbol\ell_{T}\1-\boldsymbol\ell_{T})-\eta_{T}^2(\boldsymbol\ell_T-\pi_T^T\boldsymbol\ell_T\1)^2)\Big( \exp(-\eta_{T-1}\widetilde L_{T-1})+e^{-1}\dfrac{\eta_{T-1}-\eta_T}{\eta_{T-1}}\Big)\\,.
\]
Then we use the expression
\[
\eta_T=\dfrac{\eta_{T-1}}{\sqrt{1+2.2\eta_{T-1}^2(\boldsymbol\ell_T-\pi_{T}^T\boldsymbol\ell_T\1)^2}}
\]
and the notation $x_T=\eta_{T-1}(\boldsymbol\ell_T-\pi_{T}^T\boldsymbol\ell_T\1)$ to derive
\[
\exp(-\eta_{T}\widetilde L_{T})\le \exp\Big(-\dfrac{x_T}{\sqrt{1+2.2x_T^2}}-\dfrac{x_T^2}{1+2.2 x_T^2}\Big)\Big( \exp(-\eta_{T-1}\widetilde L_{T-1})+\dfrac{\eta_{T-1}-\eta_T}{\eta_{T-1}}\Big)\,.
\]
We use different bounds over the function $\varphi:~y\in \R\mapsto \exp\Big(-\dfrac{y}{\sqrt{1+2.2y^2}}-\dfrac{y^2}{1+2.2y^2}\Big)$:\\
 $\varphi(y)\le e/2$, $\varphi(y)\le 1-\dfrac{y}{\sqrt{1+2.2y^2}}$ for any $y\in \R$ and $\varphi(y)\le 1-y$ if $y\le 1/4$.  Distinguishing whether $x_T$ is larger or not than $1/4$, we deduce 
\begin{align*}
\exp(-\eta_{T}\widetilde L_{T})\le &(\1-\eta_{T-1}(\boldsymbol\ell_T-\pi_{T}^T\boldsymbol\ell_T\1))\exp(-\eta_{T-1}\widetilde L_{T-1})\1\{{x_T\le1/4}\}\\
&+(\1-\eta_{T}(\boldsymbol\ell_T-\pi_{T}^T\boldsymbol\ell_T\1))\exp(-\eta_{T-1}\widetilde L_{T-1})\1\{{x_T>1/4}\}+1/2\dfrac{\eta_{T-1}-\eta_T}{\eta_{T-1}}\,.
\end{align*}
Using the relations $\eta_{T-1}/\eta_T\ge \1$ and $1-\dfrac{y}{\sqrt{1+y^2}}>0$, $y\in \R$ we upper bound the second term by
 \begin{align*}
\dfrac{\eta_{T-1}}{\eta_T}&(\1-\eta_{T}(\boldsymbol\ell_T-\pi_{T}^T\boldsymbol\ell_T\1))\exp(-\eta_{T-1}\widetilde L_{T-1})\1\{{x_T>1/4}\}\\
= & \Big(\dfrac{\eta_{T-1}}{\eta_T}- \eta_{T-1}(\boldsymbol\ell_T-\pi_{T}^T\boldsymbol\ell_T\1))\exp(-\eta_{T-1}\widetilde L_{T-1})\1\{{x_T>1/4}\}\,.
\end{align*}
Combining it with the previous bound we achieve
 \begin{align*}
\exp(-\eta_{T}\widetilde L_{T})\le & \Big(\dfrac{\eta_{T-1}}{\eta_T}\Big)^{\1\{{x_T>1/4}\}}\exp(-\eta_{T-1}\widetilde L_{T-1})\\
& -\eta_{T-1}(\boldsymbol\ell_T-\pi_{T}^T\boldsymbol\ell_T\1)\exp(-\eta_{T-1}\widetilde L_{T-1})+1/2\dfrac{\eta_{T-1}-\eta_T}{\eta_{T-1}}\,.
\end{align*}
The second inequality is obtained .  We have
\begin{align*}
\pi_1^T\exp(-\eta_{T}\widetilde L_{T})\le &\Big\| \Big(\dfrac{\eta_{T-1}}{\eta_T}\Big)^{\1\{{x_T>1/4}\}}\Big\|_\infty\pi_1^T\exp(-\eta_{T-1}\widetilde L_{T-1})\\
& -\Big(\pi_1\eta_{T-1}\exp(-\eta_{T-1}\widetilde L_{T-1})\Big)^T(\boldsymbol\ell_T-\pi_{T}^T\boldsymbol\ell_T\1)+1/2\pi_1^T\dfrac{\eta_{T-1}-\eta_T}{\eta_{T-1}}\,.
\end{align*}
We recognize the weights
\[
\pi_1\eta_{T-1}\exp(-\eta_{T-1}\widetilde L_{T-1})=\pi_T\big(\pi_1^T\eta_{T-1}\exp(-\eta_{T-1}\widetilde L_{T-1})\big)
\]
and the second term in the upper bound is proportional to $\pi_T^T(\boldsymbol\ell_T-\pi_{T}^T\boldsymbol\ell_T\1)=0$ and thus vanishes. We obtain
\[
\pi_1^T\exp(-\eta_{T}\widetilde L_{T})\le\Big\| \Big(\dfrac{\eta_{T-1}}{\eta_T}\Big)^{\1\{{x_T>1/4}\}}\Big\|_\infty\pi_1^T\exp(-\eta_{T-1}\widetilde L_{T-1})+1/2\pi_1^T\dfrac{\eta_{T-1}-\eta_T}{\eta_{T}}\]
and a recursive argument yields
\begin{multline*}
\pi_1^T\exp(-\eta_{T}\widetilde L_{T})\le\exp\Big(\sum_{t=2}^T\Big\|\log\Big(\dfrac{\eta_{t-1}}{ \eta_t}\Big)\1\{{x_t>1/4}\}\Big\|_\infty\Big)\Big(\pi_1^T\exp\Big(-\eta_{1}\widetilde L_{1}\Big)+1/2\sum_{t=2}^T\pi_1^T\dfrac{\eta_{T-1}-\eta_T}{\eta_{T-1}}\Big).
\end{multline*}
We bound the exponent term such as 
\begin{align*}
\sum_{t=2}^T\Big\|\log\Big(\dfrac{\eta_{t-1}}{ \eta_t}\Big)\1\{x_t>1/4\}\Big\|_\infty&\le\sum_{i=1}^K\sum_{t=2}^T\log\Big( \dfrac{\eta_{t-1,i}}{\eta_{t,i}}\Big) \1\{x_{t,i}>1/4\}\\
&\le \sum_{i=1}^K1\big\{\max_{2\le t\le T}x_{t,i}>1/4\big\} \Big(\sum_{t=2}^T\log \Big(\dfrac{\eta_{t-1,i}}{\eta_{t,i}}\Big)\1\{x_{t,i}>1/4\} \Big)\\
&\le \sum_{i=1}^K1\big\{\max_{2\le t\le T}x_{t,i}>1/4\big\} \Big( \log\Big( \dfrac{\eta_{1,i}}{\eta_{T-1,i}} \Big)+ \log\Big( \dfrac{\eta_{T-1,i}}{\eta_{T,i}} \Big)\Big)
\end{align*}
assuming with no loss of generality that if $\max_{2\le t\le T}x_{t,i}>1/4$ then it happens for the last iterate $x_{T,i} =\eta_{T-1,i}\ell_{T,i}>1/4$. Notice also that $x_{T,i}>1/4$ implies  that $\eta_{T-1,i}^{-1}\le M_{T,i}/4$. Combined with
\[
\dfrac{\eta_{T-1,i}}{\eta_{T,i}}=\sqrt{1+2.2\eta_{T-1,i}^2(\boldsymbol\ell_{T,i}-\pi_{T}^T\boldsymbol\ell_T)^2}\le \sqrt{1+2.2\eta_{1,i}^2M_{T,i}^2}\,,
\]
we obtain  
\begin{align*}
\sum_{t=2}^T\Big\|\log\Big(\dfrac{\eta_{t-1}}{ \eta_t}\Big)\1\{x_t>1/4\}\Big\|_\infty&\le \sum_{i=1}^K1\big\{\max_{2\le t\le T}x_{t,i}>1/4\big\} \Big( \log(\eta_{1,i}M_{T,i}/4)+\frac12 \log\big(1+2.2\eta_{1,i}^2M_{T,i}^2\big)\Big)\\
&\le \sum_{i=1}^K1\big\{\max_{2\le t\le T}x_{t,i}>1/4\big\} \log\Big(1+\eta_{1,i}^2M_{T,i}^2\Big)\,.
\end{align*}
We have $\exp(-\eta_{1}\widetilde L_{1})\le \exp(\1)$ using the relation  $|\eta_{1}\widetilde L_{1}|=\1$ and the comparison sum-integral
\[
\sum_{t=2}^T\dfrac{\eta_{t-1}-\eta_{t}}{\eta_{t-1}}\le \log(\eta_1/\eta_T)=\frac12\log\Big(\1+(\eta_{1 }M_{T })^2T\Big)
\]
we achieve \eqref{eq:pi0}.

\subsection{Proof of Theorem \ref{th:boa}}
From the regret bound  \eqref{eq:boa}, keeping the notation of \eqref{eq:pi0} and applying Young's inequality, we infer that for any $\eta>0$
\[
\sum_{t=1}^T\pi_t^T\boldsymbol\ell_t- \sum_{t=1}^T\boldsymbol\ell_{t,i}\le \dfrac{\eta}2\sum_{t=1}^T(\pi_t^T\boldsymbol\ell_t-\boldsymbol\ell_{t,i})^2+\dfrac{(1/1.1+\log(\pi_1^{-1})+\log(A_T))^2}{2\eta} .
\]
Plugging this bound into \eqref{eq:surrog} and identifying $\boldsymbol \ell_t= \mathbf x_t^T \nabla\ell_t(\mathbf x_t\pi_t)$ and $\widehat x_t=\mathbf x_t\pi_t$ we obtain
\begin{align*}
\sum_{t=1}^TL_t(\widehat x_t)-\sum_{t=1}^TL_t( x_t^{(i)})\le& \frac{\lambda+\eta}2\sum_{t=1}^T\nabla \ell_t(\widehat x_t)^T(\widehat x_t-x_t^{(i)})^2\nonumber\\
&+\frac{\lambda-\alpha}2\sum_{t=1}^T\E_{t-1}[(\nabla \ell_t(\widehat x_t)^T(\widehat x_t-x_t^{(i)})^2]\\
&+\dfrac{(1/1.1+\log(\pi_1^{-1})+\log(A_T))^2}{2\eta}+\frac2\lambda\log(\delta^{-1})\,.
\end{align*}
Applying once again the Poissonian inequality \eqref{eq:poisson}, using that the diameter of the simplex satisfies is less than $1$, we derive that with probability $1-\delta$
\[
\sum_{t=1}^T(\nabla \ell_t(\widehat x_t)^T(\widehat x_t-x_t^{(i)})^2\le 2\sum_{t=1}^T\E_{t-1}[ (\nabla \ell_t(\widehat x_t)^T(\widehat x_t-x_t^{(i)})^2]+2(G_{\psi_2}D)^2\log(\delta^{-1})\,.
\]
Then we obtain
\begin{align*}
\sum_{t=1}^TL_t(\widehat x_t)-\sum_{t=1}^TL_t( x_t^{(i)})\le& \frac{3\lambda+2\eta-\alpha}2\sum_{t=1}^T\E_{t-1}[(\nabla \ell_t(\widehat x_t)^T(\widehat x_t-x_t^{(i)})^2]\\
&+\dfrac{(1/1.1+\log(\pi_1^{-1})+\log(A_T))^2}{2\eta}+\Big((\lambda+\eta)(G_{\psi_2}D)^2+\frac2\lambda\Big)\log(\delta^{-1})\,.
\end{align*}
Thus choosing $\lambda=\eta=\alpha/3$ and introducing  $\nabla \ell_t(\hat x_t)$ for bounding roughly $\log(A_T)$, we obtain
\begin{multline*}
\sum_{t=1}^TL_t(\widehat x_t)-\sum_{t=1}^TL_t( x_t^{(i)})\le  \dfrac{3}\alpha \Big(K\log\Big(1+\dfrac{\max_{1\le t\le T}\|\nabla \ell_t(\hat x_t)\|^2}{\underline m^ 2}\Big)\\
+\log\Big(\ex+\log\Big(1+\dfrac{\max_{1\le t\le T}\|\nabla \ell_t(\hat x_t)\|^2}{\underline m^ 2}T\Big)\Big)+1/1.1+\log(\pi_1^{-1})\Big)^2+\Big(\dfrac{2\alpha}3(G_{\psi_2}D)^2+\frac6\alpha\Big)\log(\delta^{-1})\,.
\end{multline*}
From the proof Proposition \ref{prop:pois} on the second-order terms $0\le Y_t=\|\nabla \ell_t(\widehat x_t)\|^2/(2G_{\psi_2}^2)$ we obtain
\[
\E_{t-1}[\exp(Y_t)]\le 1+2\E_{t-1}[Y_t^2]\le 1+(G_2/G_{\psi_2})^2\,.
\] 
Thus, for any $x>0$ we have
\[
\P\Big(\max_{1\le t\le T} Y_t>x\Big)\le \E[\exp(\max Y_t)]\exp(-x)\le \sum_{t=1}^T \E[\exp( Y_t)]\exp(-x)\le T(1+(G_2/G_{\psi_2})^2)\exp(-x)
\]
and with probability $1-\delta$ it holds
\[
\max_{1\le t\le T} Y_t \le \log(T)+\log(1+(G_2/G_{\psi_2})^2)+\log(\delta^{-1})\,.
\]
Finally, we obtain the desired result using a union bound.

\subsection{Proof of Proposition \ref{prop:mean}}

We denote 
\[
\overline{y}^M_{t-1,t-p}=((y_{t-1}\wedge M/2)\vee (-M/ 2),\ldots,(y_{t-p}\wedge M/ 2)\vee (-M/ 2))\in\R^p.
\]
Let $P_x={\cal N}(\widehat m_t^{(p)}(x),\sigma^2)$ then $\ell_t(x)=(y_t-\widehat m_t^{(p)}(x))^2/(2\sigma^2)$ (plus constant) and
\begin{align*}
\E_{t-1}[\nabla \ell_t(x)\nabla \ell_t(x)^T]&=\dfrac{\E_{t-1}[(y_t-\widehat m_t^{(p)}(x))^2]}{\sigma^4}\overline{y}^M_{t-1,t-p}(\overline{y}^M_{t-1,t-p})^T\,,\\
\E_{t-1}[\nabla^2 \ell_t(x)]&=\dfrac{1}{\sigma^2}\overline{y}^M_{t-1,t-p}(\overline{y}^M_{t-1,t-p})^T\,.
\end{align*}
Because the second derivatives do not depend on $x$ a Taylor expansion provides
\begin{align*}
L_t(y)&=L_t(x)+\nabla L_t(y)^T(y-x)-\dfrac{1}{\sigma^2}(y-x)^T\overline{y}^M_{t-1,t-p}(\overline{y}^M_{t-1,t-p})^T(y-x)\\
&=L_t(x)+\nabla L_t(y)^T(y-x)-\dfrac{1}{\sigma^2}(\widehat m_t^{(p)}(y)-\widehat m_t^{(p)}(x))^2\\
&\le L_t(x)+\nabla L_t(y)^T(y-x)-\dfrac{\E_{t-1}[(y_t-\widehat m_t^{(p)}(x))^2]}{\sigma^2(\overline\sigma^2+M^2)}(\widehat m_t^{(p)}(y)-\widehat m_t^{(p)}(x))^2\\
&\le L_t(x)+\nabla L_t(y)^T(y-x)-\dfrac{\sigma^2}{\overline\sigma^2+M^2}(y-x)^T\E_{t-1}[\nabla \ell_t(x)\nabla \ell_t(x)^T](y-x)\,.
\end{align*}
The first inequality comes from the relations
\[
\E_{t-1}[(y_t-\widehat m_t^{(p)}(x))^2]=\E_{t-1}[(y_t- m_t)^2]+(m_t-\widehat m_t^{(p)}(x))^2\le \overline\sigma^2+M^2\,.
\]
Thus Condition {\bf (H2)} is satisfied with $\alpha=\sigma^2/(\overline \sigma^2+M^2)$.\\

Applying Theorem \ref{th:ons}, the ONS achieves the stochastic regret against every $x\in B_1(1)$ (satisfying $\|x\|\le \sqrt p$)
\[
\sum_{t=1}^TL_t(x_t)- \sum_{t=1}^TL_t(x)\le O\Big(\frac{\overline \sigma^2+M^2}{\sigma^2}p \log T + \Big(\frac{\overline \sigma^2+M^2}{\sigma^2}+\frac{\sigma^2}{\overline \sigma^2+M^2}p\, G_{\psi_2}^2\Big)\log(\delta^{-1}))\Big)
\]
with high probability. Since the risk satisfies the relation
\[
L_t(x)=\frac12\Big(\log(2\pi)+\log(\sigma^2)+\dfrac{(m_t-\widehat m_t^{(p)}(x))^2+\sigma^2_t}{\sigma^2}\Big)=KL(P_t,\mathcal N( m_t^{(p)}(x),\sigma^2)) + cst.\,,
\]
we obtain the desired result.

\subsection{Proof of Proposition \ref{prop:vol}}

 We denote 
\[
\overline{y}^{2,\overline \sigma}_{t-1,t-q}=(y_{t-1}^2\wedge \overline \sigma^2,\ldots,y_{t-q}^2\wedge \overline \sigma^2)\in\R^q.
\]
Let $P_x={\cal N}(0,\widehat\sigma_t^{2,(q)}(x))$ then $\ell_t(x)=(\log(\widehat\sigma_t^{2,(q)}(x))+y_t^2/\widehat\sigma_t^{2,(q)}(x))/2$ and
\begin{align*}
\E_{t-1}[\nabla \ell_t(x)\nabla \ell_t(x)^T]&=\dfrac{\E_{t-1}[(y_t^2-\widehat\sigma_t^{2,(q)}(x))^2]}{2\big(\widehat\sigma_t^{2,(q)}(x)\big)^4}\overline{y}^{2,\overline \sigma}_{t-1,t-q}(\overline{y}^{2,\overline \sigma}_{t-1,t-q})^T\,,\\
\E_{t-1}[\nabla^2 \ell_t(x)]&=\dfrac{2\sigma_t^2-\widehat\sigma_t^{2,(q)}(x)}{2\big(\widehat\sigma_t^{2,(q)}(x)\big)^3}\overline{y}^{2,\overline \sigma}_{t-1,t-q}(\overline{y}^{2,\overline \sigma}_{t-1,t-q})^T\,.
\end{align*}
Because $1/2\le \sigma_t^2/\widehat\sigma_t^{2,(q)}(x)\le 2$ under our assumptions, the second derivatives are decreasing in $\widehat\sigma_t^{2,(q)}(x)$ and thus 
 \[
 \E_{t-1}[\nabla^2 \ell_t(x)]\succeq \dfrac{2  \sigma^2_t-\overline \sigma^{2}}{2\overline\sigma^6}\overline{y}^{2,\overline \sigma}_{t-1,t-q}(\overline{y}^{2,\overline \sigma}_{t-1,t-q})^T\succeq \dfrac{c-1}{2\overline\sigma^4}\overline{y}^{2,\overline \sigma}_{t-1,t-q}(\overline{y}^{2,\overline \sigma}_{t-1,t-q})^T\,.
 \]
Combining this lower bound with a Taylor expansion, we obtain
\begin{align*}
L_t(y)&\le L_t(x)+\nabla L_t(y)^T(y-x)-\dfrac{c-1}{2\overline\sigma^4}(y-x)^T\overline{y}^{2,\overline \sigma}_{t-1,t-q}(\overline{y}^{2,\overline \sigma}_{t-1,t-q})^T(y-x)\\
&\le L_t(x)+\nabla L_t(y)^T(y-x)-\dfrac{(c-1)\E_{t-1}[(y_t^2-\widehat\sigma_t^{2,(q)}(x))^2]}{2\overline \sigma^4(3 \overline\sigma^4+\overline\sigma^4)}(\widehat\sigma_t^{2,(q)}(y)-\widehat\sigma_t^{2,(q)}(x))^2\\
&\le L_t(x)+\nabla L_t(y)^T(y-x)-\dfrac{(c-1)(c\overline\sigma^2/2)^4}{4\overline \sigma^8}(y-x)^T\E_{t-1}[\nabla \ell_t(x)\nabla \ell_t(x)^T](y-x)\,.
\end{align*}
The second inequality comes from the relations
\[
\E_{t-1}[(y_t^2-\widehat\sigma_t^{2,(q)}(x))^2]=\E_{t-1}[(y_t^2- \sigma_t^2)^2]+(\sigma_t^2-\widehat\sigma_t^{2,(q)}(x))^2\le 3\overline\sigma^4+\overline \sigma^2\,.
\]
Thus Condition {\bf (H2)} is satisfied with $\alpha=(c-1)c^42^{-6}$.\\

Applying Theorem \ref{th:ons}, the ONS achieves the stochastic regret with high probability 
against every $x\in {\cal K}$ (satisfying $\|x\|\le \sqrt q$)
\[
\sum_{t=1}^TL_t(x_t)- \sum_{t=1}^TL_t(x)\le O( q \log T +(1+qG_{\psi_2}^2)  \log(\delta^{-1})).
\]
We conclude the proof by identifying the KL divergence with $L_t$ up to additive constants.

\subsection*{Proof of Proposition \ref{prop:vol}}
We use similar arguments than in the proofs of Propositions \ref{prop:mean} and \ref{prop:vol}, keeping the same notation with 
\[
\ell_t(x)=\dfrac12\Big(\log(\widehat\sigma_t^{2,(q)}(x_{p+1:p+q}))+\dfrac{(y_t-\widehat m_t^{(p)}(x_{1:p}))^2}{\widehat\sigma_t^{2,(q)}(x_{p+1:p+q})}\Big)\,.
\]
Adapting previous computations, we similarly obtain a lower bound on the second derivatives
 \[
 \E_{t-1}[\nabla^2 \ell_t(x)] \succeq \dfrac{c-1}{2\overline\sigma^4}\big(\overline{y}^M_{t-1,t-p},\overline{y}^{2,\overline \sigma}_{t-1,t-q}\big)\big(\overline{y}^M_{t-1,t-p},\overline{y}^{2,\overline \sigma}_{t-1,t-q}\big)^T\,.
 \]
We can also upper bound the first derivatives to obtain
\[
 \E_{t-1}[\nabla \ell_t(x)\nabla \ell_t(x)^T]\preceq \dfrac{2^4}{2 (c\overline \sigma^2)^4}\E_{t-1}[(\widehat\sigma_t^{2,(q)}-(y_t-\widehat m_t^{(p)}(x_{1:p}))^2)^2] \big(\overline{y}^M_{t-1,t-p},\overline{y}^{2,\overline \sigma}_{t-1,t-q}\big)\big(\overline{y}^M_{t-1,t-p},\overline{y}^{2,\overline \sigma}_{t-1,t-q}\big)^T\,.
\]
Under our assumptions, we roughly estimate
\begin{align*}
\E_{t-1}[(\widehat\sigma_t^{2,(q)}-(y_t-\widehat m_t^{(p)}(x_{1:p}))^2)^2]&\le 2\big(\big(\widehat\sigma_t^{2,(q)}\big)^2 +\E_{t-1}\big[(y_t-\widehat m_t^{(p)}(x_{1:p}))^4\big]\big)\\
&\le 2\big(\overline\sigma^4+2\big(\E_{t-1}\big[(y_t-  m_t )^4\big]+\E_{t-1}\big[(m_t-\widehat m_t^{(p)}(x_{1:p}))^4\big]\big)\big)\\
&\le 18\overline\sigma^4\,.
\end{align*}
Thus Condition {\bf (H2)} is satisfied with $\alpha=(c-1)c^43^{-2}2^{-5}$.\\

Applying Theorem \ref{th:ons}, the ONS achieves the stochastic regret with high probability 
against every $x\in {\cal K}$ (satisfying $\|x\|\le D=\sqrt{p+q}$)
\[
\sum_{t=1}^TL_t(x_t)- \sum_{t=1}^TL_t(x)\le O( (p+q) \log T +(1+(p+q)G_{\psi_2}^2)  \log(\delta^{-1})).
\]
We conclude the proof by identifying the KL divergence with $L_t$ up to additive constants.

\vspace*{-10pt}

\section{Auxiliary results}

\subsection{The stopping time argument of \cite{ville1939etude} and \cite{freedman1975tail}}\label{app:freedman}
We recall the argument of \cite{ville1939etude} and \cite{freedman1975tail} as we apply it several times in the proofs of the paper. Consider $M_T=\exp(\sum_{t=1}^TZ_t)$ for any $Z_t$ adapted to a filtration $\mathcal F_t$ and satisfying the exponential inequality $\E[\exp(Z_t)\mid \mathcal F_{t-1}]\le 1$. Then we have
\[
\P\Big(\exists T\ge 1 : \sum_{t=1}^TZ_t>\log(\delta^{-1})\Big)\le \delta
\]
for any $0<\delta<1$ by applying the following lemma.

\begin{lemma}
If $M_t$ is adapted to $\mathcal F_t$, $M_0=1$ a.s. and $\E[M_t\mid \mathcal F_{t-1}]\le M_{t-1}$ a.s., $t\ge1$, then, for any $0<\delta<1$, it holds 
\[
\P(\exists T\ge 1:\; M_T >\delta^{-1})\le \delta\,.
\]
\end{lemma}
\begin{proof}
We apply the optional stopping theorem with Markov's inequality defining the stopping time $\tau = \inf\{t>1:\, M_t>\delta^{-1}\}$ so that
\[
\P(\exists t\ge 1:\, M_t >\delta^{-1})=\P(M_\tau>\delta^{-1})\le \E[M_\tau] \delta\le \E[M_0] \delta\le \delta \,.
\]
\end{proof}

\subsection{SOCO analysis of the OGD algorithm}\label{sec:ogd}

In this section we work under {\bf (H1)} and {\bf (H2)}  with $\alpha=0$. Proposition \ref{prop:surrogate} holds, $\lambda>0=\alpha$ and the compensator term in Proposition \ref{prop:surrogate} is positive. In this section we assume that the gradients are bounded by $G<\infty$. A slow rate stochastic regret bound $O(GD\sqrt T)$ is expected and the surrogate loss in Proposition \ref{prop:surrogate} is useless. The classical Online Gradient Descent (OGD) of \cite{zinkevich2003online}
\[
x_{t+1}=\arg\min_{x\in{\cal K}}\Big\|x-\frac{D}{G\sqrt t}\,\nabla\,\ell_t(x_t)\Big\|\qquad \mbox{starting from }x_0\in\cal K\,,
\]
satisfies the following linearized regret bound in any SOCO problem, see the proof in \cite{hazan2016introduction} that does not use any convex assumption,
\[
 \sum_{t=1}^T\nabla \ell_t(x_t)^T(x_t-x)\le \frac32 DG\sqrt T\,.
\]
Under {\bf (H1)} we easily bound a.s. both extra quadratic terms in Proposition \ref{prop:surrogate} with the same quantity $\lambda /2\,G^2D^2 T$. Choosing $\lambda = \sqrt{2\log(\delta^{-1})}/(GD \sqrt T)$ we immediately obtain a new slow rate stochastic regret bound for the OGD valid in any SOCO problem:
\begin{theorem}\label{th:ogd}
Assume that {\bf (H1)} holds and that $\sup_{x\in \cal K}\|\nabla \ell_t(x)\|\le G$ a.s., $t\ge1$. The OGD algorithm satisfies with probability $1-\delta$ the stochastic regret bound
\[
\sum_{t=1}^TL_t(x_t)-\sum_{t=1}^TL_t(x)\le \Big(\frac32+2\sqrt{2 \log(\delta^{-1})}\Big) DG\sqrt T
\]
valid for any $T\ge 1$ and any $x\in {\cal K}$.
\end{theorem}
This simple extension of the usual iid setting to any stochastic adversarial setting could be obtained by classical arguments such as Azuma's inequality used in Chapter 9 of \cite{hazan2016introduction}. It relies on the martingale $\sum_{t=1}^T(\nabla L_t(x_t)-\nabla \ell_t(x_t))^T(x_t-x^*)$ and the gradient trick on $L_t$ to remove the assumption of convexity on the losses $\ell_t$.

\end{document}